\documentclass[twoside,11pt]{article}
\usepackage{pslatex}
\usepackage{journal}
\usepackage{amsmath,amssymb} 
\usepackage{bbm}
\usepackage{times}

\usepackage{tabls}
\usepackage{tabularx}
\usepackage{booktabs}
\usepackage{color}
\usepackage{breakurl}

\renewenvironment{proof}[1][]{\par\noindent{\bf Proof #1\ }}{\hfill\BlackBox\\[2mm]}

\newcommand{\E}{\mathbb{E}}
\renewcommand{\P}{\mathbb{P}}
\newcommand{\VS}{{\rm VS}}
\newcommand{\F}{\mathcal{F}}
\newcommand{\dc}{\theta}
\newcommand{\Ball}{{\rm B}}
\newcommand{\DIS}{{\rm DIS}}
\newcommand{\PDIS}{\Delta}
\newcommand{\VScomp}{\hat{\mathcal{C}}}
\newcommand{\ind}{\mathbbm{1}}
\newcommand{\target}{f^{*}}
\newcommand{\cX}{{\cal X}}
\newcommand{\cY}{{\cal Y}}

\newcommand{\eqdef}{\triangleq}
\newcommand{\reals}{\mathbb{R}}
\newcommand{\CAL}{{\rm CAL}}
\newcommand{\LC}{\Lambda}
\newcommand{\Dset}{\mathbb{D}}
\newcommand{\nats}{\mathbb{N}}
\newcommand{\polylog}{{\rm polylog}}
\newcommand{\poly}{{\rm poly}}
\newcommand{\er}{{\rm er}}
\renewcommand{\L}{\mathcal{L}}
\renewcommand{\H}{\mathcal{H}}
\newcommand{\U}{\mathcal{U}}
\newcommand{\sign}{{\rm sign}}
\newcommand{\Bound}[3]{\mathcal{B}_{#1}\!\left(#2,#3\right)} 
\newcommand{\hatn}[1]{\hat{n}(S_{#1})}

\newcommand{\argmax}{\mathop{\rm argmax}}
\newcommand{\argmin}{\mathop{\rm argmin}}

\newtheorem{condition}[theorem]{Condition}

\newsavebox{\savepar}
\newenvironment{bigboxit}{\begin{center}\begin{lrbox}{\savepar}
\begin{minipage}[h]{4.5in}
\normalfont
\begin{flushleft}}
{\end{flushleft}\end{minipage}\end{lrbox}\fbox{\usebox{\savepar}}
\end{center}}

\jmlrheading{xx}{2014}{xx-xx}{3/14}{x/xx}{Yair Wiener, Steve Hanneke, and Ran El-Yaniv}

\ShortHeadings{Active Learning}{Wiener, Hanneke, and El-Yaniv}
\firstpageno{1}

\begin{document}
\title{A Compression Technique for Analyzing Disagreement-Based \\Active Learning}

\author{%
\name Yair Wiener 
\email yair.wiener@gmail.com\\
\addr \addr Department of Computer Science\\
Technion -- Israel Institute of Technology\\
\name Steve Hanneke 
\email steve.hanneke@gmail.com\\
\name Ran El-Yaniv
\email rani@cs.technion.ac.il\\
\addr Department of Computer Science\\
Technion -- Israel Institute of Technology}

\editor{?}

\maketitle

\begin{abstract}
We introduce a new and improved characterization of the label complexity of disagreement-based active learning,
in which the leading quantity is the \emph{version space compression set size}. 
This quantity is defined as the size 
of the smallest subset of the training data that induces the same version space.
We show various applications of the new characterization, including 
a tight analysis of CAL and refined label complexity bounds for linear separators under 
mixtures of Gaussians and axis-aligned rectangles under product densities.
The version space compression set size, as well as the new characterization of the label complexity, 
can be naturally extended to agnostic learning problems, for which we show new 
speedup results for two well known active learning algorithms.
\end{abstract}


\section{Introduction}
\label{sec:intro}

Active learning is a learning paradigm allowing the learner to sequentially
request the target labels of selected instances from a pool or stream of unlabeled
data.\footnote{Any active learning technique for streaming data can be used in pool-based
models but not vice versa}
The key question in the theoretical analysis of active learning is how many
label requests are sufficient to learn the labeling function to a specified
accuracy, a quantity known as the \emph{label complexity}.
Among the many recent advances in the theory of active learning,
perhaps the most well-studied technique has been the \emph{disagreement-based}
approach, initiated by \citet*{CohAtlLad94}, and further advanced in numerous articles
\citep*[e.g.,][]{balcan2009agnostic,dasgupta2007general,Beygelzimer08,Beygelzimer10,koltchinskii:10,Hanneke11,hanneke:12b}.
The basic strategy in disagreement-based active learning is to sequentially process the
unlabeled examples, and for each example, the algorithm
requests its label if and only if the value of the optimal classifier's classification
on that point cannot be inferred from information already obtained.

One attractive feature of this approach is that its simplicity makes it amenable to
thorough theoretical analysis, and numerous theoretical guarantees on the performance
of variants of this strategy under various conditions have appeared in the literature
\citep*[see e.g.,][]{balcan2009agnostic,Hanneke07,dasgupta2007general,balcan:07,Beygelzimer08,Friedman09,BHV:10,hannekeAnnals11,koltchinskii:10,Beygelzimer10,hsu:thesis,Hanneke11,el2012active,hanneke:12b,hannekestatistical}.
The majority of these results formulate bounds on the label complexity in terms
of a complexity measure known as the \emph{disagreement coefficient} \citep*{Hanneke07},
which we define below.	A notable exception to this is the recent work of 
\citet*{el2012active}, rooted in the related topic of selective prediction \citep*{ElYaniv10,wiener2012pointwise,Wiener13},
which instead bounds the label complexity in terms of two complexity measures called the
\emph{characterizing set complexity} and the \emph{version space compression set size} \citep*{ElYaniv10}.
In the current literature,
the above are the only known general techniques for the analysis of disagreement-based active learning.

In the present article, we present a new characterization of the label complexity of
disagreement-based active learning.
The leading quantity in our characterization is the \emph{version space compression set size}
of \citet*{el2012active,ElYaniv10,Wiener13},
which corresponds to the size of the smallest subset of the training set that induces the same version space as the entire training set. 
This complexity measure was shown by \citet*{el2012active} to be a special case of the extended teaching dimension
of \citet*{Hanneke_COLT_07}.

The new characterization improves upon the two prior techniques in some cases.
For a noiseless setting (the realizable case), we show that the label complexity results derived from this new technique
are \emph{tight} up to logarithmic factors.  This was not true of either
of the previous techniques; as we discuss in Appendix~\ref{app:looseness-examples},
the known upper bounds in the literature expressed in terms of
these other complexity measures are sometimes off by a factor of the VC dimension.
Moreover, the new method significantly simplifies the recent technique
of \citet*{Wiener13,el2012active,ElYaniv10}
by completely eliminating the need
for the characterizing set complexity
measure.

Interestingly, interpreted as an upper bound on the label complexity of active learning in general,
the upper bounds presented here also reflect improvements over a bound of \citet*{Hanneke_COLT_07},
which is also expressed in terms of (a target-independent variant of) this same complexity measure:
specifically, reducing the bound by roughly a factor of the VC dimension compared to that result.
In addition to these results on the label complexity, we also relate the version space compression set size
to the disagreement coefficient, essentially showing that they are always within a factor of
the VC dimension of each other (with additional logarithmic factors).

We apply this new technique to derive new results for two learning problems: namely, 
linear separators under mixtures of Gaussians, and axis-aligned hyperrectangles under product densities.  
We derive bounds on the version space compression set size for each of these.
Thus, using our results relating the version space compression set size to the label complexity,
we arrive at bounds on the label complexity of disagreement-based active learning for these problems,
which represent significant refinements of the best results in the prior literature on these settings.

While the version space compression set size is initially defined for noiseless (realizable) 
learning problems that have a version space, it can be naturally extended to an agnostic setting,
and the new technique applies to noisy, agnostic problems as well. 
This surprising result, 
which was motivated by related observations of \citet*{hannekestatistical,Wiener13},
is allowed through bounds on the disagreement coefficient in terms of the version space compression set size,
and the applicability of the disagreement coefficient to both the realizable and agnostic settings.
We formulate this generalization in Section~\ref{sec:noise} and present new sample complexity results
for known active learning
algorithms, 
including the disagreement-based methods of \citet*{dasgupta2007general} and 
\citet*{Hanneke11}.
These results tighten the bounds of \citet*{Wiener13} using the new technique.

\section{Preliminary Definitions}
\label{sec:definitions}

Let $\cX$ denote a set, called the \emph{instance space},
and let $\cY \eqdef \{-1,+1\}$, called the \emph{label space}.
A \emph{classifier} is a measurable function $h : \cX \to \cY$.
Throughout, we fix a set $\F$ of classifiers, called the \emph{concept space},
and denote by $d$ the VC dimension of $\F$ \citep*{VapniC71,Vapnik98}.
We also fix an arbitrary probability measure $P$ over $\cX \times \cY$,
called the \emph{data distribution}.   Aside from Section~\ref{sec:noise},
we make the assumption that $\exists \target \in \F$ with $\P(Y=\target(x)|X=x)=1$
for all $x \in \cX$, where $(X,Y) \sim P$; this is known as the \emph{realizable case},
and $\target$ is known as the \emph{target function}.
For any classifier $h$, define its \emph{error rate} $\er(h) \eqdef P( (x,y) : h(x) \neq y )$;
note that $\er(\target) = 0$.

For any set $\H$ of classifiers, define the \emph{region of disagreement}
\begin{equation*}
\DIS(\H) \eqdef \{x \in \cX : \exists h,g \in \H \text{ s.t. } h(x) \neq g(x)\}.
\end{equation*}
Also define $\PDIS \H \eqdef P( \DIS(\H) \times \cY )$, the marginal probability of the region of disagreement.

Let $S_{\infty} \eqdef \{(x_1,y_1),(x_2,y_2),\ldots\}$ be a sequence of i.i.d. $P$-distributed random variables,
and for each $m \in \nats$, denote by $S_{m} \eqdef \{(x_1,y_1),\ldots,(x_m,y_m)\}$.\footnote{Note 
that, in the realizable case, $y_{i} = \target(x_{i})$ for all $i$ with probability $1$.
For simplicity, we will suppose these equalities hold throughout our discussion of the realizable case.}
For any $m \in \nats \cup \{0\}$, and any $S \in (\cX \times \cY)^{m}$,
define the \emph{version space} $\VS_{\F,S} \eqdef \{h \in \F : \forall (x,y) \in S, h(x) = y\}$ \citep*{Mitchell77}.
The following definition will be central in our results below.

\begin{definition}[Version Space Compression Set Size]
For any $m \in \nats \cup \{0\}$ and any $S \in (\cX \times \cY)^{m}$,
the \emph{version space compression set} $\VScomp_{S}$
is a smallest subset of $S$ satisfying $\VS_{\F,\VScomp_{S}} = \VS_{\F,S}$.
The \emph{version space compression set size} is defined
to be $\hat{n}(\F,S) \eqdef |\VScomp_{S}|$.
In the special cases where $\F$ and perhaps $S = S_{m}$ are obvious from the context,
we abbreviate $\hat{n} \eqdef \hatn{m} \eqdef \hat{n}(\F,S_{m})$. 
\end{definition}

Note that the value $\hat{n}(\F,S)$ is unique for any $S$, and $\hatn{m}$ is, obviously, 
a random number that depends on the (random) sample $S_{m}$.
The quantity $\hatn{m}$ has been studied under at least two names in the prior literature.
Drawing motivation from the work on Exact learning with Membership Queries \citep*{Hegedus95,hellerstein:96},
which extends ideas from \citet*{GolKea95} on the complexity of teaching,
the quantity $\hatn{m}$ was introduced in the work of \citet*{Hanneke_COLT_07} as the
\emph{extended teaching dimension} of the classifier $\target$ on the space $\{x_{1},\ldots,x_{m}\}$
with respect to the set $\F[\{x_1,\ldots,x_m\}]\eqdef\{ x_i \mapsto h(x_i) : h \in \F\}$ of distinct classifications of $\{x_1,\ldots,x_m\}$ realized by $\F$;
in this context, the set $\VScomp_{S_{m}}$ is known as a
\emph{minimal specifying set} of $\target$ on $\{x_{1},\ldots,x_{m}\}$ with respect to $\F[\{x_1,\ldots,x_m\}]$.
The quantity $\hatn{m}$ was independently discovered by \citet*{ElYaniv10} in the context of
selective classification, which is the source of the compression set terminology introduced above;
we adopt this terminology throughout the present article.  See the work of \citet*{el2012active}
for a formal proof of the equivalence of these two notions.

It will also be useful to define minimal confidence bounds on certain quantities.
Specifically, for any $m \in \nats \cup \{0\}$ and $\delta \in (0,1]$,
define the \emph{version space compression set size minimal bound}
\begin{equation}
\label{eq:vscompBound}
\Bound{\hat{n}}{m}{\delta} \eqdef \min\left\{ b \in \nats \cup \{0\} : \P( \hatn{m} \leq b ) \geq 1-\delta \right\}.
\end{equation}
Similarly, define the \emph{version space disagreement region minimal bound}
\begin{equation*}
\Bound{\PDIS}{m}{\delta} \eqdef \min\left\{ t \in [0,1] : \P( \PDIS \VS_{\F,S_{m}} \leq t ) \geq 1-\delta \right\}.
\end{equation*}
In both cases, the quantities implicitly also depend on $\F$ and $P$ (which remain fixed throughout our analysis below),
and the only random variables involved in 
these probabilities are the data $S_{m}$.

Most of the existing general results 
on disagreement-based active learning are expressed in terms
of a quantity known as the \emph{disagreement coefficient} \citep*{Hanneke07,hannekethesis09},
defined as follows.
\begin{definition}[Disagreement Coefficient]
For any classifier $f$ and $r > 0$, define the $r$-ball centered at $f$ as
\begin{equation*}
\Ball(f,r) \eqdef \left\{ h \in \F : \PDIS \{h,f\} \leq r \right\},
\end{equation*}
and for any $r_{0} \geq 0$, define the \emph{disagreement coefficient} of $\F$ with respect to $P$ as\footnote{We use the notation $a \lor b = \max\{a,b\}$.}
\begin{equation*}
\dc(r_{0}) \eqdef \sup_{r > r_{0}} \frac{\PDIS \Ball(\target,r)}{r} \lor 1.
\end{equation*}
\end{definition}

The disagreement coefficient was originally introduced to the active learning literature by \citet*{Hanneke07},
and has been studied and bounded by a number of authors \citep*[see e.g.,][]{Hanneke07,Friedman09,Wang2011,hannekestatistical,balcan:13}.
Similar quantities have also been studied in the passive learning literature,
rooted in the work of Alexander \citep*[see e.g.,][]{alexander:87,gine:06}.

Numerous recent results, many of which are surveyed by \citet*{hannekestatistical}, exhibit bounds on the label complexity of
disagreement-based active learning in terms of the disagreement coefficient. It is therefore of major interest to develop such
bounds for specific cases of interest (i.e., for specific classes $\F$ and distributions $P$).
In particular, any result showing $\dc(r_{0}) = o(1/r_{0})$ indicates that disagreement-based active learning
should asymptotically provide some advantage over passive learning for that $\F$ and $P$ \citep*{Hanneke11}.
We are particularly interested in scenarios in which $\dc(r_{0}) = O(\polylog(1/r_{0}))$, or even $\dc(r_{0}) = O(1)$,
since these imply strong improvements over passive learning \citep*{Hanneke07,hannekeAnnals11}.

There are several general results on the asymptotic behavior of the disagreement coefficient 
as $r_{0} \to 0$, for interesting cases.
For the class of linear separators in $\reals^{k}$, perhaps the most general result to date
is that the existence of a density function for the marginal distribution of
$P$ over $\cX$ is sufficient to guarantee $\dc(r_0) = o(1/r_0)$ \citep*{hannekestatistical}. 
That work also shows that, if the density is bounded and has bounded support, and the target separator passes
through the support at a continuity point of the density, then $\dc(r_{0}) = O(1)$.
In both of these cases, for $k \geq 2$, the specific dependence on $r_0$ in the little-$o$ and the constant factors in the big-$O$ 
will vary depending on the particular distribution $P$, and in particular, will depend on $\target$
(i.e., such bounds are \emph{target-dependent}).

There are also several explicit, \emph{target-independent} bounds on the disagreement coefficient in the literature.
Perhaps the most well-known of these is for homogeneous linear separators in $\reals^{k}$,
where the marginal distribution of $P$ over $\cX$ is confined to be the uniform distribution over the unit sphere, 
in which case $\dc(r_{0})$ is known to be within a factor of $4$ of $\min\{ \pi\sqrt{k}, 1/r_{0} \}$ \citep*{Hanneke07}.
In the present paper, we are primarily focused on explicit, target-independent speedup bounds,
though our abstract results can be used to derive bounds of either type.

\section{Relating $\mathbf{\hat{n}}$ and the Disagreement Coefficient}
\label{sec:td-dc}

In this section, we show how to bound the disagreement coefficient in terms of $\Bound{\hat{n}}{m}{\delta}$.
We also show the other direction and bound $\Bound{\hat{n}}{m}{\delta}$ in terms of the disagreement 
coefficient. 

\begin{theorem}
\label{thm:td-dc-bound}
For any $r_{0} \in (0,1)$,
\begin{equation*}
\dc(r_{0}) \leq \max\left\{ \max_{r \in (r_{0},1)} 16 \Bound{\hat{n}}{\left\lceil \frac{1}{r}\right\rceil}{\frac{1}{20}}, 512\right\}.
\end{equation*}
\end{theorem}
\begin{proof}
We will prove that, for any $r \in (0,1)$,
\begin{equation}
\label{eqn:dc-bound-nosup}
\frac{\PDIS \Ball(\target, r)}{r} \leq \max\left\{ 16 \Bound{\hat{n}}{\left\lceil \frac{1}{r} \right\rceil}{\frac{1}{20}}, 512\right\}.
\end{equation}
The result then follows by taking the supremum of both sides over $r \in (r_{0},1)$.

Fix $r \in (0,1)$, let $m = \lceil 1 / r \rceil$,
and for $i \in \{1,\ldots,m\}$, define $S_{m \setminus i} = S_{m} \setminus \{(x_{i},y_{i})\}$.
Also define $D_{m \setminus i} = \DIS( \VS_{\F,S_{m \setminus i}} \cap \Ball(\target,r) )$
and $\PDIS_{m \setminus i} = \P(x_{i} \in D_{m \setminus i} | S_{m \setminus i}) = P( D_{m \setminus i} \times \cY )$.
%
If $\PDIS \Ball(\target,r) m \leq 512$, \eqref{eqn:dc-bound-nosup} clearly holds.
Otherwise, suppose $\PDIS \Ball(\target,r) m > 512$.
If $x_{i} \in \DIS( \VS_{\F,S_{m \setminus i}} )$, then we must have $(x_{i},y_{i}) \in \VScomp_{S_{m}}$.
So
\begin{equation*}
\hatn{m}
\geq \sum_{i=1}^{m} \ind_{\DIS(\VS_{\F,S_{m \setminus i}})}(x_{i}).
\end{equation*}
Therefore,
\begin{align*}
& \P\left\{ \hatn{m} \leq (1/16) \PDIS \Ball(\target,r) m \right\}
\\ & \leq \P\left\{ \sum_{i=1}^{m} \ind_{\DIS(\VS_{\F,S_{m \setminus i}})}(x_{i}) \leq (1/16) \PDIS \Ball(\target,r) m \right\}
\\ & \leq \P\left\{ \sum_{i=1}^{m} \ind_{D_{m \setminus i}}(x_{i}) \leq (1/16) \PDIS \Ball(\target,r) m \right\}
\\ & = \P\left\{ \sum_{i=1}^{m} \ind_{\DIS(\Ball(\target,r))}(x_{i}) - \ind_{D_{m \setminus i}}(x_{i}) \geq \sum_{i=1}^{m} \ind_{\DIS(\Ball(\target,r))}(x_{i}) - (1/16) \PDIS \Ball(\target,r) m \right\}
\\
& =
\P \left\{  \sum_{i=1}^m  \ind_{\DIS( \Ball( \target, r))}(x_i) - \ind_{D_{m \setminus i}}(x_{i})
	 \geq  \right.\\
	 &  \ \ \ \ \ \ \ \ \ \ \left. \sum_{i=1}^m  \ind_{\DIS( \Ball( \target, r))}(x_i) - \frac{1}{16} \PDIS\Ball(\target, r) m
	 , \ \ \
	\sum_{i=1}^m  \ind_{\DIS( \Ball( \target, r))}(x_i) <  \frac{7}{8} \PDIS \Ball(\target, r)m
	\right\} \\
	& +
	\P \left\{
	\sum_{i=1}^m  \ind_{\DIS( \Ball( \target, r))}(x_i)  - \ind_{D_{m \setminus i}}(x_{i})
	\geq
	\right. \\
	& \ \ \ \ \ \ \ \ \ \  \left. \sum_{i=1}^m  \ind_{\DIS( \Ball( \target, r))}(x_i)  - \frac{1}{16} \PDIS \Ball(\target, r) m
	, \ \ \  \sum_{i=1}^m \ind_{\DIS( \Ball( \target, r))}(x_i) \geq \frac{7}{8} \PDIS \Ball (\target, r) m
	\right\} \\
& \leq
\P\left\{ \sum_{i=1}^{m} \ind_{\DIS(\Ball(\target,r))}(x_{i}) < (7/8) \PDIS \Ball(\target,r) m \right\}
\\ & \ \ \ \ \ + \P\left\{ \sum_{i=1}^{m} \ind_{\DIS(\Ball(\target,r))}(x_{i}) - \ind_{D_{m \setminus i}}(x_{i}) \geq (13/16) \PDIS \Ball(\target,r) m \right\}.
\end{align*}

Since we are considering the case $\PDIS \Ball(\target,r) m > 512$,
a Chernoff bound implies
\begin{equation*}
\P\left( \sum_{i=1}^{m} \ind_{\DIS(\Ball(\target,r))}(x_{i}) < (7/8) \PDIS \Ball(\target,r) m \right)
\leq \exp\left\{ - \PDIS \Ball(\target,r) m / 128 \right\} < e^{-4}.
\end{equation*}
Furthermore, Markov's inequality implies
\begin{equation*}
\P\left( \sum_{i=1}^{m} \ind_{\DIS(\Ball(\target,r))}(x_{i}) - \ind_{D_{m \setminus i}}(x_{i}) \geq (13/16) \PDIS \Ball(\target,r) m \right)
\leq \frac{ m \PDIS \Ball(\target,r) - \E\left[ \sum_{i=1}^{m} \ind_{D_{m \setminus i}}(x_{i}) \right] }{(13/16) m \PDIS \Ball(\target,r)}.
\end{equation*}
Since the $x_{i}$ values are exchangeable,
\begin{equation*}
\E\left[ \sum_{i=1}^{m} \ind_{D_{m \setminus i}}(x_{i}) \right]
= \sum_{i=1}^{m} \E\left[ \E\left[ \ind_{D_{m \setminus i}}(x_{i}) \Big| S_{m \setminus i} \right] \right]
= \sum_{i=1}^{m} \E\left[ \PDIS_{m \setminus i} \right]
= m \E\left[\PDIS_{m \setminus m}\right].
\end{equation*}
\citet*{Hanneke11} proves that this is at least
\begin{equation*}
m (1-r)^{m-1} \PDIS \Ball(\target,r).
\end{equation*}
In particular, when $\PDIS \Ball(\target,r) m > 512$, we must have $r < 1/511 < 1/2$,
which implies $(1-r)^{\lceil 1/r \rceil - 1}$ $\geq 1/4$,
so that we have
\begin{equation*}
\E\left[ \sum_{i=1}^{m} \ind_{D_{m \setminus i}}(x_{i}) \right]
\geq (1/4) m \PDIS \Ball(\target,r).
\end{equation*}
Altogether, we have established that
\begin{equation*}
\P\left( \hatn{m} \leq (1/16) \PDIS \Ball(\target,r) m \right)
< \frac{ m \PDIS \Ball(\target,r) - (1/4) m \PDIS \Ball(\target,r) }{(13/16) m \PDIS \Ball(\target,r)}
+ e^{-4}
= \frac{12}{13} + e^{-4}
< \frac{19}{20}.
\end{equation*}
Thus, since $\hatn{m} \leq \Bound{\hat{n}}{m}{\frac{1}{20}}$
with probability at least $\frac{19}{20}$, we must have that
\begin{equation*}
\Bound{\hat{n}}{m}{\frac{1}{20}}
> (1/16) \PDIS \Ball(\target,r) m
\geq (1/16) \frac{ \PDIS \Ball(\target,r) }{r}. 
\end{equation*}
\end{proof}

The following Theorem, whose proof is given in Section~\ref{sec:cal},
is a ``converse'' of Theorem~\ref{thm:td-dc-bound},
showing a bound on $\Bound{\hat{n}}{m}{d}$ in terms of the disagreement coefficient.

\begin{theorem}
\label{thm:vcdc-td-bound}
There is a finite universal constant $c > 0$ such that,
$\forall r_{0},\delta \in (0,1)$,
\begin{equation*}
\max_{r \in (r_{0},1)} \Bound{\hat{n}}{\left\lceil \frac{1}{r} \right\rceil}{\delta}
\leq c \dc(d r_{0}) \left( d \ln( e\dc(d r_{0}) ) + \ln\left( \frac{ \log_{2}(2/r_{0}) }{\delta} \right) \right) \log_{2}\left( \frac{2}{r_{0}} \right).
\end{equation*}
\end{theorem}

\section{Tight Analysis of CAL}
\label{sec:cal}

The following algorithm is due to \citet*{CohAtlLad94}. 

\begin{bigboxit}
Algorithm: \textbf{\CAL}($n$)
\\ 0. $m \gets 0$, $t \gets 0$, $V_{0} \gets \F$
\\ 1. While $t < n$
\\ 2. \quad $m \gets m+1$
\\ 3. \quad If $x_{m} \in \DIS(V_{m-1})$
\\ 4. \qquad Request label $y_m$; let $V_{m} \gets \{h \in V_{m-1} : h(x_m) = y_m\}$, $t \gets t+1$
\\ 5. \quad Else $V_{m} \gets V_{m-1}$
\\ 6. Return any $\hat{h} \in V_{m}$
\end{bigboxit}

One particularly attractive feature of this algorithm is that it maintains the invariant that $V_{m} = \VS_{\F,S_{m}}$ for all values of $m$ it obtains
(since, if $V_{m-1} = \VS_{\F,S_{m-1}}$,  then $\target \in V_{m-1}$, so any point $x_{m} \notin \DIS(V_{m-1})$ has $\{h \in V_{m-1} : h(x_m) = y_m\} = \{h \in V_{m-1} : h(x_m) = \target(x_m)\} = V_{m-1}$ anyway).
To analyze this method, we first define, for every $m \in \nats$,
\begin{equation*}
N(m;S_{m}) = \sum_{t=1}^{m} \ind_{\DIS(\VS_{\F,S_{t-1}})}(x_{t}),
\end{equation*}
which counts the number of labels requested by CAL among the first $m$ data points (assuming it does not halt first).
The following result provides data-dependent upper and lower bounds on this important quantity,
which will be useful in establishing label complexity bounds for \CAL~below.

\begin{lemma}
\label{lem:data-dependent-cal-queries}
\begin{equation*}
\max_{t \leq m} \hatn{t} \leq N(m;S_{m}),
\end{equation*}
and with probability at least $1-\delta$,
\begin{equation*}
N(m;S_{m}) \leq \max_{t \in \{ 2^{i} : i \in \{0,\ldots,\lfloor \log_{2}(m) \rfloor\}\}} \left( 55 \hatn{t} \ln\left( \frac{e t}{\hatn{t}} \right) + 24 \ln\left(\frac{4 \log_{2}(2m)}{\delta}\right) \right) \log_{2}(2m).
\end{equation*}
\end{lemma}

Since the upper and lower bounds on $N(m;S_{m})$ in Lemma~\ref{lem:data-dependent-cal-queries} require access to the \emph{labels} of the data,
they are not as much interesting for practice as they are for their theoretical significance.  In particular, they will allow us to derive new
distribution-dependent bounds on the performance of CAL below (Theorems~\ref{thm:cal-queries-confidence-bound} and~\ref{thm:cal-label-complexity}).
Lemma~\ref{lem:data-dependent-cal-queries} is also of some \emph{conceptual} significance, as it shows a direct and fairly-tight
connection between the behavior of \CAL~and the size of the version space compression set.

The proof of the upper bound on $N(m;S_{m})$ relies on the following two lemmas.  
The first lemma (Lemma~\ref{lem:compression}) is implied by a classical compression bound of \citet*{littlestone:86},
and provides a high-confidence bound on
the probability measure of a set, given that it has zero empirical frequency and is specified 
by a small number of samples.
For completeness, we include a proof of this result below: a variant of the original argument of \citet*{littlestone:86}.\footnote{See 
also Section~5.2.1 of \citet*{herbrich:02} for a very clear and concise proof 
of a similar result (beginning with the line above (5.15) there, for our purposes).}

\begin{lemma}[Compression; \citealp*{littlestone:86}] 
\label{lem:compression}
For any $\delta \in (0,1)$, any collection $\Dset$ of measurable sets $D \subseteq \cX \times \cY$,
any $m \in \nats$ and $n \in \nats \cup \{0\}$ with $n \leq m$, and any permutation-invariant function $\phi_{n} : (\cX \times \cY)^{n} \to \Dset$,
with probability of at least $1-\delta$ over draw of $S_m$, 
every distinct $i_{1},\ldots,i_{n} \in \{1,\ldots,m\}$ with 
$S_{m} \cap \phi_{n}((x_{i_1},y_{i_1}),\ldots,(x_{i_n},y_{i_n})) = \emptyset$
satisfies\footnote{We define $0 \ln(1/0) = 0 \ln(\infty) = 0$.}
\begin{equation}
\label{eqn:compression-prob}
P(\phi_{n}((x_{i_1},y_{i_1}),\ldots,(x_{i_n},y_{i_n}))) \leq \frac{1}{m-n}\left(n \ln\left( \frac{e m}{n} \right) + \ln\left( \frac{1}{\delta} \right)\right).
\end{equation}
\end{lemma}

\newcommand{\ii}{\mathbf{i}}
\newcommand{\jj}{\mathbf{j}}

\begin{proof}
Let $\epsilon > 0$ denote the value of the right hand side of \eqref{eqn:compression-prob}.
The result trivially holds if $\epsilon > 1$.  For the remainder, consider the case $\epsilon \leq 1$.
Let $I_n$ be the set of all sets of $n$ distinct indices $\{i_{1},\ldots,i_{n}\}$ from $\{1,\ldots, m\}$.
Note that $|I_n| = \binom{m}{n}$. 
Given a labeled sample $S_m$ and $\ii = \{i_{1},\ldots,i_{n}\} \in I_n$, denote by $S_m^{\ii} = \{(x_{i_{1}},y_{i_{1}}),\ldots,(x_{i_{n}},y_{i_{n}})\}$,
and by $S_m^{-\ii} = \{(x_{i},y_{i}) : i \in \{1,\ldots,m\} \setminus \ii\}$.
%
Since $\phi_{n}$ is permutation-invariant, for any distinct $i_{1},\ldots,i_{n} \in \{1,\ldots,m\}$,
letting $\ii = \{i_{1},\ldots,i_{n}\}$ denote the unordered set of indices, we may denote 
$\phi_{n}( S_{m}^{\ii} ) = \phi_{n}( (x_{i_{1}},y_{i_{1}}),\ldots,(x_{i_{n}},y_{i_{n}}) )$ without ambiguity.
In particular, we have 
$\{ \phi_{n}( (x_{i_{1}},y_{i_{1}}),\ldots,(x_{i_{n}},y_{i_{n}}) ) : i_{1},\ldots,i_{n} \in \{1,\ldots,m\} \text{ distinct} \} = \{ \phi_{n}( S_{m}^{\ii} ) : \ii \in I_{n} \}$,
so that it suffices to show that, with probability at least $1-\delta$, every $\ii \in I_{n}$ with $S_{m} \cap \phi_{n}(S_{m}^{\ii}) = \emptyset$ 
has $P(\phi_{n}(S_{m}^{\ii})) \leq \epsilon$.

Define the events $\omega(\ii, m) = \left\{ S_{m} \cap \phi_{n}(S_m^{\ii}) = \emptyset \right\}$
and $\omega^{\prime}(\ii, m-n) = \left\{ S_{m}^{-\ii} \cap \phi_{n}(S_m^{\ii}) = \emptyset \right\}$.
Note that $\omega(\ii,m) \subseteq \omega^{\prime}(\ii,m-n)$.
Therefore, for each $\ii \in I_{n}$, we have 
\begin{equation*}
\P\left( \left\{  P(\phi_{n}( S_{m}^{\ii} )) > \epsilon \right\} \cap \omega(\ii,m) \right)
\leq
\P\left( \left\{ P(\phi_{n}( S_{m}^{\ii} )) > \epsilon \right\} \cap \omega^{\prime}(\ii,m-n) \right).
\end{equation*}
By the law of total probability and $\sigma(S_{m}^{\ii})$-measurability of the event $\left\{ P(\phi_{n}(S_{m}^{\ii})) > \epsilon \right\}$, this equals
\begin{equation*}
\E\left[ \P\left( \left\{ P(\phi_{n}( S_{m}^{\ii} )) > \epsilon \right\} \cap \omega^{\prime}(\ii,m-n) \Big| S_{m}^{\ii} \right) \right]
= \E\left[ \ind[ P(\phi_{n}( S_{m}^{\ii} )) > \epsilon ] \P\left( \omega^{\prime}(\ii,m-n)\Big| S_{m}^{\ii} \right) \right].
\end{equation*}
Noting that $| S_{m}^{-\ii} \cap \phi_{n}(S_{m}^{\ii}) |$ is conditionally ${\rm {Binomial}}(m-n, P(\phi_{n}(S_{m}^{\ii})))$ given $S_{m}^{\ii}$,
this equals
\begin{equation*}
\E\left[ \ind[ P(\phi_{n}( S_{m}^{\ii} )) > \epsilon ] \left( 1 - P(\phi_{n}(S_{m}^{\ii})) \right)^{m-n} \right]
\leq (1-\epsilon)^{m-n}
\leq e^{-\epsilon (m-n)},
\end{equation*}
where the last inequality is due to $1-\epsilon \leq e^{-\epsilon}$ \citep*[see e.g., Theorem A.101 of][]{herbrich:02}.
In the case $n=0$, this last expression equals $\delta$, which establishes the result since $|I_{0}|=1$.  Otherwise, if $n > 0$,
combining the above with a union bound, we have that
\begin{multline*}
\P\left( \exists \ii \in I_{n} : P(\phi_{n}( S_{m}^{\ii} )) > \epsilon \land S_{m} \cap \phi_{n}(S_{m}^{\ii}) = \emptyset \right)
= \P\left( \bigcup_{\ii \in I_{n}} \left\{ P(\phi_{n}( S_{m}^{\ii} )) > \epsilon \right\} \cap \omega(\ii,m) \right)
\\ \leq \sum_{\ii \in I_{n}} \P\left( \left\{ P(\phi_{n}( S_{m}^{\ii} )) > \epsilon \right\} \cap \omega(\ii,m) \right)
\leq \sum_{\ii \in I_{n}} e^{-\epsilon (m-n)}
=\binom{m}{n} e^{-\epsilon (m-n)}.
\end{multline*}
Since $\binom{m}{n} \leq \left( \frac{e m}{n} \right)^{n}$ \citep*[see e.g., Theorem A.105 of][]{herbrich:02}, 
this last expression is at most $\left( \frac{e m}{n} \right)^{n} e^{-\epsilon (m-n)} = \delta$, which completes the proof.
\end{proof}

The following, Lemma~\ref{lem:data-dependent-coverage},
will be used for proving
Lemma~\ref{lem:data-dependent-cal-queries} above.
The lemma relies on Lemma~\ref{lem:compression}
and provides a high-confidence bound on the probability of requesting the next label at any given point
in the \CAL~algorithm. This refines a related result of \citet*{ElYaniv10}.
Lemma~\ref{lem:data-dependent-coverage} is also of independent interest in the context of
selective prediction \citep*{Wiener13,ElYaniv10}, as it can be used to
improve the known coverage bounds for realizable selective classification.

\begin{lemma}
\label{lem:data-dependent-coverage}
For any $\delta \in (0,1)$ and $m \in \nats$, with probability at least $1-\delta$,
\begin{equation*}
\PDIS \VS_{\F,S_{m}} \leq \frac{10\hatn{m} \ln\left(\frac{e m}{\hatn{m}}\right)  + 4\ln\left(\frac{2}{\delta}\right)}{m}.
\end{equation*}
\end{lemma}
\begin{proof}
The proof is similar to that of a result of \citet*{ElYaniv10}, except using a generalization
bound based directly on sample compression, rather than the VC dimension.
Specifically, let $\Dset = \{ \DIS( \VS_{\F,S} ) \times \cY : S \in (\cX \times \cY)^{m} \}$,
and for each $n \leq m$ and $S \in (\cX \times \cY)^{n}$, let $\phi_{n}(S) = \DIS( \VS_{\F,S} ) \times \cY$.
In particular, note that for any $n \geq \hatn{m}$, any superset $S$ of $\VScomp_{S_{m}}$ of size $n$ contained in $S_{m}$ has
$\phi_{n}(S) = \DIS(\VS_{\F,S_{m}}) \times \cY$, and therefore $S_{m} \cap \phi_{n}(S) = \emptyset$ and $\PDIS \VS_{\F,S_{m}} = P( \phi_{n}(S) )$.
Therefore, Lemma~\ref{lem:compression} implies that, for each $n \in \{0,\ldots,m\}$,
with probability at least $1-\delta / (n+2)^{2}$, if $\hatn{m} \leq n$,
\begin{equation*}
\PDIS \VS_{\F,S_{m}} \leq \frac{1}{m-n} \left( n \ln\left( \frac{e m}{n} \right) + \ln\left( \frac{(n+2)^{2}}{\delta} \right) \right).
\end{equation*}
Furthermore, since $\PDIS \VS_{\F,S_{m}} \leq 1$, any $n \geq m/2$ trivially has $\PDIS \VS_{\F,S_{m}} \leq 2n/m \leq (2/m)( n \ln(e m / n) + \ln((n+2)^{2}/\delta))$, 
while any $n \leq m/2$ has $1/(m-n) \leq 2/m$,
so that the above is at most
\begin{equation*}
\frac{2}{m} \left( n \ln\left( \frac{e m}{n} \right) + \ln\left( \frac{(n+2)^{2}}{\delta} \right) \right).
\end{equation*}
Additionally, $\ln( (n+2)^{2} ) \leq 2 \ln (2) + 4 n  \leq 2\ln(2) + 4 n \ln ( e m / n )$, so that the above is at most
\begin{equation*}
\frac{2}{m} \left( 5 n \ln\left( \frac{e m}{n} \right) + 2\ln\left( \frac{2}{\delta} \right) \right).
\end{equation*}
By a union bound, this holds for all $n \in \{0,\ldots,m\}$ with probability at least $1-\sum_{n=0}^{m} \delta / (n+2)^{2} > 1-\delta$.
In particular, since $\hatn{m}$ is always in $\{0,\ldots,m\}$, this implies the result.
\end{proof}
\begin{proof}[of Lemma~\ref{lem:data-dependent-cal-queries}]
For any $t \leq m$, by definition of $\hat{n}$ (in particular, minimality), \emph{any} set $S \subset S_{t}$ with $|S| < \hatn{t}$
necessarily has $\VS_{\F,S} \neq \VS_{\F,S_{t}}$.  Thus, since $\CAL$ maintains that $V_{t}=\VS_{\F,S_{t}}$, and $V_{t}$ is precisely the set
of classifiers in $\F$ that are correct on the $N(t;S_{t})$ points $(x_i,y_i)$ with $i \leq t$ for which $\ind_{\DIS(\VS_{\F,S_{i-1}})}(x_i)=1$, we must
have $N(t;S_{t}) \geq \hatn{t}$.  We therefore have $\max_{t \leq m} \hatn{t} \leq \max_{t \leq m} N(t;S_{t}) = N(m;S_{m})$ (by monotonicity of $t \mapsto N(t;S_{t})$).

For the upper bound, let $\delta_{i}$ be a sequence of values in $(0,1]$ with $\sum_{i=0}^{\lfloor \log_{2}(m) \rfloor} \delta_{i} \leq \delta/2$.
Lemma~\ref{lem:data-dependent-coverage} implies that, for each $i$, with probability at least $1-\delta_{i}$,
\begin{equation*}
\PDIS \VS_{\F,S_{2^{i}}} \leq 2^{-i} \left( 10\hatn{2^{i}} \ln\left( \frac{e 2^{i}}{\hatn{2^{i}}} \right)  + 4\ln\left(\frac{2}{\delta_{i}}\right) \right).
\end{equation*}
Thus, by monotonicity of $\PDIS \VS_{\F,S_{t}}$ in $t$, a union bound implies that with probability at least $1-\delta/2$,
for every $i \in \{0,1,\ldots,\lfloor \log_{2}(m) \rfloor\}$, every $t \in \{2^{i},\ldots,2^{i+1}-1\}$ has
\begin{equation}
\label{eqn:data-dependent-cal-queries-DeltaVS-bound}
\PDIS \VS_{\F,S_{t}} \leq 2^{-i} \left( 10\hatn{2^{i}} \ln\left( \frac{e 2^{i}}{\hatn{2^{i}}} \right)  + 4\ln\left(\frac{2}{\delta_{i}}\right) \right).
\end{equation}
Noting that $\left\{ \ind_{\DIS(\VS_{\F,S_{t-1}})}(x_{t}) - \PDIS \VS_{\F,S_{t-1}} \right\}_{t=1}^{\infty}$ is a martingale difference
sequence with respect to $\{x_{t}\}_{t=1}^{\infty}$, Bernstein's inequality (for martingales) implies that with probability
at least $1-\delta/2$, if \eqref{eqn:data-dependent-cal-queries-DeltaVS-bound} holds for all $i \in \{0,1,\ldots,\lfloor \log_{2}(m) \rfloor\}$ and $t \in \{2^{i},\ldots,2^{i+1}-1\}$,
then
\begin{multline*}
\sum_{t=1}^{m} \ind_{\DIS(\VS_{\F,S_{t-1}})}(x_{t})
\leq 1 + \sum_{i=0}^{\lfloor \log_{2}(m) \rfloor} \sum_{t=2^{i}+1}^{2^{i+1}} \ind_{\DIS(\VS_{\F,S_{2^{i}}})}(x_{t})
\\ \leq \log_{2}\left( \frac{4}{\delta} \right) + 2 e \sum_{i=0}^{\lfloor \log_{2}(m) \rfloor} \left( 10\hatn{2^{i}} \ln\left( \frac{e 2^{i}}{\hatn{2^{i}}} \right)  + 4\ln\left(\frac{2}{\delta_{i}}\right) \right).
\end{multline*}
Letting $\delta_{i} = \frac{\delta}{2 \lfloor \log_{2}(2m) \rfloor}$, the above is at most
\begin{equation*}
\max_{i \in \{0,1,\ldots,\lfloor \log_{2}(m) \rfloor\}} \left( 55 \hatn{2^{i}} \ln\left( \frac{e 2^{i}}{\hatn{2^{i}}} \right) + 24 \ln\left(\frac{4 \log_{2}(2m)}{\delta}\right) \right) \log_{2}(2m).
\end{equation*}
\end{proof}

This also implies distribution-dependent bounds on any confidence bound on the number of queries made by \CAL.
Specifically, let $\Bound{N}{m}{\delta}$ be the smallest nonnegative integer $n$ such that $\P( N(m;S_{m}) \leq n ) \geq 1-\delta$.
Then the following result follows immediately from Lemma~\ref{lem:data-dependent-cal-queries}.

\begin{theorem}
\label{thm:cal-queries-confidence-bound}
For any $m \in \nats$ and $\delta \in (0,1)$,
for any sequence $\delta_{t}$ in $(0,1]$ with $\sum_{i=0}^{\lfloor \log_{2}(m) \rfloor} \delta_{2^{i}} \leq \delta/2$,
\begin{multline*}
\max_{t \leq m} \Bound{\hat{n}}{t}{\delta}
\leq \Bound{N}{m}{\delta}
\\ \leq \max_{t \in \{ 2^{i} : i \in \{0,1,\ldots,\lfloor \log_{2}(m) \rfloor\}\}} \left( 55 \Bound{\hat{n}}{t}{\delta_{t}} \ln\left( \frac{e t}{\Bound{\hat{n}}{t}{\delta_{t}}} \right) + 24 \ln\left(\frac{8 \log_{2}(2m)}{\delta}\right) \right) \log_{2}(2m).
\end{multline*}
\end{theorem}
\begin{proof}
Since Lemma~\ref{lem:data-dependent-cal-queries} implies every $t \leq m$ has $\hatn{t} \leq N(m;S_{m})$,
we have
$\P( \hatn{t} \leq \Bound{N}{m}{\delta} ) \geq \P( N(m;S_{m}) \leq \Bound{N}{m}{\delta} ) \geq 1-\delta$.
Since $\Bound{\hat{n}}{t}{\delta}$ is the smallest $n \in \nats$ with $\P( \hatn{t} \leq n ) \geq 1-\delta$,
we must therefore have $\Bound{\hat{n}}{t}{\delta} \leq \Bound{N}{m}{\delta}$, from which the left inequality
in the claim follows by maximizing over $t$.

For the second inequality, the upper bound on $N(m;S_{m})$ from Lemma~\ref{lem:data-dependent-cal-queries}
implies that, with probability at least $1-\delta/2$, $N(m;S_{m})$ is at most
\begin{equation*}
\max_{t \in \{ 2^{i} : i \in \{0,\ldots,\lfloor \log_{2}(m) \rfloor\}\}} \left( 55 \hatn{t} \ln\left(\frac{e t}{\hatn{t}}\right) + 24 \ln\left(\frac{8 \log_{2}(2m)}{\delta}\right) \right) \log_{2}(2m).
\end{equation*}
Furthermore, a union bound implies that with probability at least $1-\sum_{i=0}^{\lfloor \log_{2}(m) \rfloor} \delta_{2^{i}} \geq 1-\delta/2$,
every $t \in \{ 2^{i} : i \in \{0,\ldots,\lfloor \log_{2}(m) \rfloor\} \}$ has $\hatn{t} \leq \Bound{\hat{n}}{t}{\delta_{t}}$.
Since $x \mapsto x \ln(e t / x)$ is nondecreasing for $x \in [0,t]$, and $\Bound{\hat{n}}{t}{\delta_{t}} \leq t$,
combining these two results via a union bound, we have that with probability at least $1-\delta$, $N(m;S_{m})$ is at most
\begin{equation*}
\max_{t \in \{ 2^{i} : i \in \{0,1,\ldots,\lfloor \log_{2}(m) \rfloor\}\}} \left( 55 \Bound{\hat{n}}{t}{\delta_{t}} \ln\left( \frac{e t}{\Bound{\hat{n}}{t}{\delta_{t}}} \right) + 24 \ln\left(\frac{8 \log_{2}(2m)}{\delta}\right) \right) \log_{2}(2m).
\end{equation*}
Letting $U_{m}$ denote this last quantity, note that since $N(m;S_{m})$ is a nonnegative integer,
$N(m;S_{m}) \leq U_{m} \Rightarrow N(m;S_{m}) \leq \lfloor U_{m} \rfloor$, so that $\P( N(m;S_{m}) \leq \lfloor U_{m} \rfloor ) \geq 1-\delta$.
Since $\Bound{N}{m}{\delta}$ is the \emph{smallest} nonnegative integer $n$ with $\P( N(m;S_{m}) \leq n ) \geq 1-\delta$,
we must have $\Bound{N}{m}{\delta} \leq \lfloor U_{m} \rfloor \leq U_{m}$.
\end{proof}

In bounding the label complexity of CAL,
we are primarily interested in the size of $n$ sufficient to guarantee low error rate for every classifier in
the final $V_{m}$ set (since $\hat{h}$ is taken to be an arbitrary element of $V_{m}$).  Specifically, we are
interested in the following quantity.
For $n \in \nats$, define $M(n;S_{\infty}) = \min\{ m \in \nats : N(m;S_{m}) = n\}$ (or $M(n;S_{\infty}) = \infty$ if $\max_{m} N(m;S_{m}) < n$),
and for any $\epsilon,\delta \in (0,1]$, define
\begin{equation*}
\LC(\epsilon,\delta) = \min\left\{ n \in \nats : \P\left(  \sup_{h \in \VS_{\F,S_{M(n;S_{\infty})}}} \er(h) \leq \epsilon \right) \geq 1-\delta \right\}.
\end{equation*}
Note that, for any $n \geq \LC(\epsilon,\delta)$, with probability at least $1-\delta$, the classifier $\hat{h}$ produced by $\CAL(n)$ has $\er(\hat{h}) \leq \epsilon$.
Furthermore, for any $n < \LC(\epsilon,\delta)$, with probability greater than $\delta$, there exists a choice of $\hat{h}$ in the final step of $\CAL(n)$ for which $\er(\hat{h}) > \epsilon$.
Therefore, in a sense, $\LC(\epsilon,\delta)$ represents the label complexity of the general family of CAL strategies (which vary only in how $\hat{h}$ is chosen from the final $V_{m}$ set).
We can also define an analogous quantity for passive learning by empirical risk minimization:
\begin{equation*}
M(\epsilon,\delta) = \min\left\{ m \in \nats : \P\left( \sup_{h \in \VS_{\F,S_{m}}} \er(h) \leq \epsilon \right) \geq 1-\delta \right\}.
\end{equation*}
We typically expect $M(\epsilon,\delta)$ to be larger than $\Omega(1/\epsilon)$, 
and it is known $M(\epsilon,\delta)$ is always at most $O((1/\epsilon)(d\log(1/\epsilon) + \log(1/\delta)))$ \citep*[e.g.,][]{Vapnik98}.
We have the following theorem relating these two quantities.

\begin{theorem}
\label{thm:cal-label-complexity}
There exists a universal constant $c \in (0,\infty)$ such that,
$\forall \epsilon,\delta \in (0,1)$,
$\forall \beta \in \left(0,\frac{1-\delta}{\delta}\right)$,
for any sequence $\delta_{m}$ in $(0,1]$ with $\sum_{i=0}^{\lfloor \log_{2}(M(\epsilon,\delta/2)) \rfloor} \delta_{2^{i}} \leq \delta/2$,
\begin{multline*}
\max_{m \leq M(\epsilon,1-\beta\delta)} \Bound{\hat{n}}{m}{(1+\beta)\delta}
\leq \LC(\epsilon,\delta)
\\ \leq c \left( \max_{m \leq M(\epsilon,\delta/2)} \Bound{\hat{n}}{m}{\delta_{m}} \ln\left( \frac{e m}{\Bound{\hat{n}}{m}{\delta_{m}}} \right) + \ln\left(\frac{\log_{2}(2M(\epsilon,\delta/2))}{\delta}\right) \right) \log_{2}(2M(\epsilon,\delta/2)).
\end{multline*}
\end{theorem}
\begin{proof}
By definition of $M(\epsilon,1-\beta\delta)$, $\forall m < M(\epsilon,1-\beta\delta)$,
with probability greater than $1-\beta\delta$, $\sup_{h \in \VS_{\F,S_{m}}} \er(h) > \epsilon$.
Furthermore, by definition of $\Bound{\hat{n}}{m}{(1+\beta)\delta}$, $\forall n < \Bound{\hat{n}}{m}{(1+\beta)\delta}$, 
with probability greater than $(1+\beta)\delta$, $\hatn{m} > n$, which together with Lemma~\ref{lem:data-dependent-cal-queries}
implies $N(m;S_{m}) > n$, so that $M(n;S_{\infty}) < m$.  Thus, fixing any $m \leq M(\epsilon,1-\beta\delta)$ and $n < \Bound{\hat{n}}{m}{(1+\beta)\delta}$,
a union bound implies that with probability exceeding $\delta$,
$M(n;S_{\infty}) < m$ and $\sup_{h \in \VS_{\F,S_{m-1}}} \er(h) > \epsilon$.
By monotonicity of $t \mapsto \VS_{\F,S_{t}}$, this implies that with probability greater than $\delta$,
$\sup_{h \in \VS_{\F,S_{M(n;S_{\infty})}}} \er(h) > \epsilon$, so that $\LC(\epsilon,\delta) > n$.

For the upper bound, Lemma~\ref{lem:data-dependent-cal-queries} and a union bound imply that, with probability at least $1-\delta/2$,
\begin{multline*}
N(M(\epsilon,\delta/2);S_{M(\epsilon,\delta/2)})
\leq
\\ c^{\prime} \left( \max_{m \leq M(\epsilon,\delta/2)} \Bound{\hat{n}}{m}{\delta_{m}} \ln\left( \frac{e m}{\Bound{\hat{n}}{m}{\delta_{m}}} \right) + \ln\left(\frac{\log_{2}(2M(\epsilon,\delta/2))}{\delta}\right) \right) \log_{2}(2M(\epsilon,\delta/2)),
\end{multline*}
for a universal constant $c^{\prime} > 0$.  In particular, this implies that for any $n$ at least this large,
with probability at least $1-\delta/2$, $M(n+1;S_{\infty}) \geq M(\epsilon,\delta/2)$.
Furthermore, by definition of $M(\epsilon,\delta/2)$ and monotonicity of $m \mapsto \sup_{h \in \VS_{\F,S_{m}}} \er(h)$, with probability at least $1-\delta/2$,
every $m \geq M(\epsilon,\delta/2)$ has $\sup_{h \in \VS_{\F,S_{m}}} \er(h) \leq \epsilon$.
By a union bound, with probability at least $1-\delta$, $\sup_{h \in \VS_{\F,S_{M(n+1;S_{\infty})}}} \er(h) \leq \epsilon$.
This implies $\LC(\epsilon,\delta) \leq n+1$, so that the result holds (for instance, it suffices to take $c=c^{\prime}+2$).
\end{proof}

For instance $\delta_{m} = \delta / \left(2\log_{2}(2M(\epsilon,\delta/2))\right)$ might be a natural choice in the above result.


Another implication of these results is a complement to Theorem~\ref{thm:td-dc-bound} that was presented in Theorem~\ref{thm:vcdc-td-bound} above.

%
\begin{proof}[of Theorem~\ref{thm:vcdc-td-bound}]
Lemma~\ref{lem:dc-N-bound} in Appendix~\ref{sec:dc-appendix} 
and monotonicity of $\epsilon \mapsto \dc(\epsilon)$ imply that, for $m = \lceil 1/r_{0} \rceil$,
\begin{align*}
\Bound{N}{m}{\delta} & \leq 8 \lor c_{0} \dc(d r_{0} / 2) \left( d \ln(e\dc(d r_{0} / 2)) + \ln\left( \frac{\log_{2}(2/r_{0})}{\delta} \right) \right) \log_{2}\left( \frac{2}{r_{0}} \right)
\\ & \leq (c_{0} \lor 8) \dc(d r_{0} / 2) \left( d \ln(e\dc(d r_{0} / 2)) + \ln\left( \frac{\log_{2}(2/r_{0})}{\delta} \right) \right) \log_{2}\left( \frac{2}{r_{0}} \right),
\end{align*}
for a finite universal constant $c_{0} > 0$.
The result then follows from Theorem~\ref{thm:cal-queries-confidence-bound}
and the fact that $\dc(d r_{0} / 2) \leq 2 \dc(d r_{0})$ \citep*{hannekestatistical}.
\end{proof}

This also implies the following corollary on the necessary and sufficient conditions
for \CAL~to provide exponential improvements in label complexity when passive learning by
empirical risk minimization has $\Omega(1/\epsilon)$ sample complexity (which is typically the case).\footnote{All
of these equivalences continue to hold even when this $M(\epsilon,\cdot)=\Omega(1/\epsilon)$ condition fails, excluding
statements \ref{item:a1} and \ref{item:a2}, which would then be implied by the others but not vice versa.}

\begin{corollary}[Characterization of CAL]
\label{cor:exponential-rates}
If $d < \infty$, and $\exists \delta_{0} \in (0,1)$ such that $M(\epsilon,\delta_{0}) = \Omega(1/\epsilon)$,
then the following are all equivalent:
\begin{enumerate}
\item \label{item:a1} $\LC(\epsilon,\delta) = O\left( \polylog\left( \frac{1}{\epsilon} \right) \log\left(\frac{1}{\delta}\right)\right)$, 
\item \label{item:a2} $\LC\!\left(\epsilon,\frac{1}{40}\right) = O\left( \polylog\left( \frac{1}{\epsilon} \right) \right)$, 
\item \label{item:b1} $\Bound{\hat{n}}{m}{\delta} = O\left( \polylog( m ) \log\left(\frac{1}{\delta}\right) \right)$, 
\item \label{item:b2} $\Bound{\hat{n}}{m}{\frac{1}{20}} = O\left( \polylog( m ) \right)$, 
\item \label{item:c} $\dc(r_{0}) = O\left( \polylog\left( \frac{1}{r_{0}} \right) \right)$, 
\item \label{item:d1} $\Bound{\PDIS}{m}{\delta} = O\left( \frac{\polylog( m )}{m} \log\left(\frac{1}{\delta}\right)\right)$, 
\item \label{item:d2} $\Bound{\PDIS}{m}{\frac{1}{9}} = O\left( \frac{\polylog( m )}{m}\right)$, 
\item \label{item:e1} $\Bound{N}{m}{\delta} = O\left( \polylog(m) \log\left(\frac{1}{\delta}\right) \right)$, 
\item \label{item:e2} $\Bound{N}{m}{\frac{1}{20}} = O\left( \polylog(m) \right)$, 
\end{enumerate}
where $\F$ and $P$ are considered constant, so that the big-$O$ hides $(\F,P)$-dependent constant factors here (but no factors depending on $\epsilon$, $\delta$, $m$, or $r_{0}$).\footnote{In fact,
we may choose freely whether or not to allow the big-$O$ to hide $\target$-dependent constants, or $P$-dependent constants in general,
as long as the \emph{same} interpretation is used for all of these statements.
Though validity for each of these interpretations generally does not imply validity for the others,
the proof remains valid regardless of which of these interpretations we choose, as long as we stick to the same interpretation throughout the proof.}
\end{corollary}
\begin{proof}
We decompose the proof into a series of implications.
Specifically, we show that
\ref{item:b1}
$\Rightarrow$ \ref{item:b2}
$\Rightarrow$ \ref{item:c}
$\Rightarrow$ \ref{item:e1}
$\Rightarrow$ \ref{item:b1},
\ref{item:e1}
$\Rightarrow$ \ref{item:e2}
$\Rightarrow$ \ref{item:b2},
\ref{item:c}
$\Rightarrow$ \ref{item:a1}
$\Rightarrow$ \ref{item:a2}
$\Rightarrow$ \ref{item:b2},
and
\ref{item:b1}
$\Rightarrow$ \ref{item:d1}
$\Rightarrow$ \ref{item:d2}
$\Rightarrow$ \ref{item:c}.
These implications form a strongly connected directed graph, and therefore establish equivalence of the statements.
\paragraph{(\ref{item:b1} $\boldsymbol{\Rightarrow}$ \ref{item:b2})}
If $\Bound{\hat{n}}{m}{\delta} = O\left( \polylog( m ) \log\left(\frac{1}{\delta}\right)\right)$,
then in particular there is some (sufficiently small) constant $\delta_{1} \in (0,1/20)$ for which
$\Bound{\hat{n}}{m}{\delta_{1}} = O\left( \polylog( m ) \right)$,
and since $\delta \mapsto \Bound{\hat{n}}{m}{\delta}$ is nonincreasing,
$\Bound{\hat{n}}{m}{\frac{1}{20}} \leq \Bound{\hat{n}}{m}{\delta_{1}}$,
so that $\Bound{\hat{n}}{m}{\frac{1}{20}} = O\left( \polylog(m) \right)$ as well.
%
\paragraph{(\ref{item:b2} $\boldsymbol{\Rightarrow}$ \ref{item:c})}
If $\Bound{\hat{n}}{m}{\frac{1}{20}} = O\left( \polylog( m ) \right)$,
then
\begin{equation*}
\max_{m \leq 1/r_{0}} \Bound{\hat{n}}{m}{\frac{1}{20}} = O\left( \max_{m \leq 1/r_{0}} \polylog( m ) \right) = O\left( \polylog\left(\frac{1}{r_{0}}\right) \right).
\end{equation*}
Therefore, Theorem~\ref{thm:td-dc-bound} implies
\begin{multline*}
\dc(r_{0}) \leq \max\left\{ \max_{m \leq \lceil 1/r_{0} \rceil} 16 \Bound{\hat{n}}{m}{\frac{1}{20}}, 512\right\}
\\ \leq 528 + 16 \max_{m \leq 1/r_{0}} \Bound{\hat{n}}{m}{\frac{1}{20}}
= O\left( \polylog\left( \frac{1}{r_{0}} \right) \right).
\end{multline*}
%
\paragraph{(\ref{item:c} $\boldsymbol{\Rightarrow}$ \ref{item:e1})}
If $\dc(r_{0}) = O\left( \polylog\left(\frac{1}{r_{0}}\right) \right)$, then Lemma~\ref{lem:dc-N-bound} in Appendix~\ref{sec:dc-appendix} implies that
$\Bound{N}{m}{\delta} = O\left( \polylog( m ) \log\left( \frac{1}{\delta} \right) \right)$.
%
\paragraph{(\ref{item:e1} $\boldsymbol{\Rightarrow}$ \ref{item:b1})}
If $\Bound{N}{m}{\delta} = O\left(\polylog(m)\log\left(\frac{1}{\delta}\right)\right)$,
then Theorem~\ref{thm:cal-queries-confidence-bound} implies
\begin{equation*}
\Bound{\hat{n}}{m}{\delta}
\leq \Bound{N}{m}{\delta}
= O\left( \polylog( m ) \log\left(\frac{1}{\delta}\right) \right).
\end{equation*}
%
\paragraph{(\ref{item:e1} $\boldsymbol{\Rightarrow}$ \ref{item:e2})}
If $\Bound{N}{m}{\delta} = O\left(\polylog(m)\log\left(\frac{1}{\delta}\right)\right)$,
then for any sufficiently small value $\delta_{2} \in (0,1/20)$, $\Bound{N}{m}{\delta_{2}} = O(\polylog(m))$;
monotonicity of $\delta \mapsto \Bound{N}{m}{\delta}$ further implies $\Bound{N}{m}{\frac{1}{20}} \leq \Bound{N}{m}{\delta_{2}}$,
so that $\Bound{N}{m}{\frac{1}{20}} = O(\polylog(m))$.
%
\paragraph{(\ref{item:e2} $\boldsymbol{\Rightarrow}$ \ref{item:b2})}
When $\Bound{N}{m}{\frac{1}{20}} = O(\polylog(m))$,
Theorem~\ref{thm:cal-queries-confidence-bound} implies that
$\Bound{\hat{n}}{m}{\frac{1}{20}} \leq \Bound{N}{m}{\frac{1}{20}} = O\left(\polylog( m )\right)$.
%
\paragraph{(\ref{item:c} $\boldsymbol{\Rightarrow}$ \ref{item:a1})}
If $\dc(r_{0}) = O\left( \polylog\left(\frac{1}{r_{0}}\right) \right)$, then Lemma~\ref{lem:dc-LC-bound} in Appendix~\ref{sec:dc-appendix} implies that
$\LC(\epsilon,\delta) = O\left( \polylog\left(\frac{1}{\epsilon}\right) \log\left(\frac{1}{\delta}\right)\right)$.
%
\paragraph{(\ref{item:a1} $\boldsymbol{\Rightarrow}$ \ref{item:a2})}
If $\LC(\epsilon,\delta) = O\left( \polylog\left(\frac{1}{\epsilon}\right) \log\left(\frac{1}{\delta}\right)\right)$,
then for any sufficiently small value $\delta_{3} \in (0,1/40]$,
$\LC(\epsilon,\delta_{3}) = O\left( \polylog\left(\frac{1}{\epsilon}\right) \right)$;
furthermore, monotonicity of $\delta \mapsto \LC(\epsilon,\delta)$ implies
$\LC\left(\epsilon,\frac{1}{40}\right) \leq \LC(\epsilon,\delta_{3})$, so that
$\LC\left(\epsilon,\frac{1}{40}\right) = O\left( \polylog\left(\frac{1}{\epsilon}\right) \right)$ as well.
%
\paragraph{(\ref{item:a2} $\boldsymbol{\Rightarrow}$ \ref{item:b2})}
Let $c \in (0,1]$ and $\epsilon_{0} \in (0,1)$ be constants such that, $\forall \epsilon \in (0,\epsilon_{0})$,
$M(\epsilon, \delta_{0}) \geq \frac{c}{\epsilon}$.
For any $\delta \in (0,1/20)$, if $\frac{19}{20}+\delta \leq \delta_{0}$, then
$M\left(\epsilon, \frac{19}{20}+\delta\right) \geq M(\epsilon,\delta_{0}) \geq c/\epsilon$;
otherwise, if $\frac{19}{20}+\delta > \delta_{0}$, then
letting $m = M(\epsilon,\frac{19}{20}+\delta)$ and $\L_{i} = \{(x_{m(i-1)+1},y_{m(i-1)+1}),\ldots,(x_{mi},y_{mi})\}$ for $i \in \nats$,
we have that $\forall k \in \nats$,
\begin{align*}
\P\left( \sup_{h \in \VS_{\F,S_{mk}}} \er(h) > \epsilon \right)
& \leq \P\left( \min_{i \leq k} \sup_{h \in \VS_{\F,\L_{i}}} \er(h) > \epsilon \right)
\\ & = \prod_{i=1}^{k} \P\left( \sup_{h \in \VS_{\F,\L_{i}}} \er(h) > \epsilon \right)
\leq \left( \frac{19}{20}+\delta\right)^{k},
\end{align*}
so that setting $k = \left\lceil \frac{\ln(1/\delta_{0})}{\ln(1/(\frac{19}{20}+\delta))} \right\rceil$
reveals that
\begin{equation}
\label{eqn:a2-b2-M-bound}
M(\epsilon,\delta_{0}) \leq M\left(\epsilon, \frac{19}{20}+\delta\right) \left\lceil \frac{\ln(1/\delta_{0})}{\ln(1/(\frac{19}{20}+\delta))} \right\rceil.
\end{equation}
Since $\ln(x) < x-1$ for $x \in (0,1)$, we have
$\ln(1/(\frac{19}{20}+\delta)) = - \ln( \frac{19}{20}+\delta ) > - (\frac{19}{20}+\delta - 1) = \frac{1}{20}-\delta$;
together with the fact that $\frac{1}{20}-\delta < 1$,
this implies
\begin{align*}
\left\lceil \frac{\ln(1/\delta_{0})}{\ln(1/(\frac{19}{20}+\delta))} \right\rceil
& \leq \left\lceil \frac{\ln(1/\delta_{0})}{\frac{1}{20}-\delta} \right\rceil
< \frac{\ln(1/\delta_{0})}{\frac{1}{20}-\delta} + 1
\\ & < \frac{\ln(1/\delta_{0})}{\frac{1}{20}-\delta} + \frac{1}{\frac{1}{20}-\delta}
= \frac{\ln(e/\delta_{0})}{\frac{1}{20}-\delta}.
\end{align*}
Plugging this into \eqref{eqn:a2-b2-M-bound} reveals that
\begin{equation*}
M\left(\epsilon, \frac{19}{20}+\delta\right)
\geq \frac{\frac{1}{20}-\delta}{\ln(e/\delta_{0})} M(\epsilon,\delta_{0})
\geq \frac{c (\frac{1}{20}-\delta)}{\ln(e/\delta_{0})} \frac{1}{\epsilon}.
\end{equation*}
If $\LC\left(\epsilon,\frac{1}{40}\right) = O\left( \polylog\left( \frac{1}{\epsilon} \right) \right)$,
then Theorem~\ref{thm:cal-label-complexity} (with $\beta = \frac{1}{20\delta}-1$ and $\delta = 1/40$) implies
\begin{equation*}
\max_{t \leq \frac{c / 40}{\ln(e/\delta_{0})} \frac{1}{\epsilon}} \Bound{\hat{n}}{t}{\frac{1}{20}} \leq \LC\left(\epsilon,\frac{1}{40}\right) = O\left( \polylog\left( \frac{1}{\epsilon} \right) \right).
\end{equation*}
This implies that, $\forall m \in \nats$,
\begin{align*}
\Bound{\hat{n}}{m}{\frac{1}{20}}
& \leq \LC\left(\frac{c / 40}{m \ln(e/\delta_{0})},\frac{1}{40}\right)
\\ & = O\left( \polylog\left( \frac{m \ln(e/\delta_{0})}{(c / 40)} \right) \right) = O\left( \polylog( m ) \right).
\end{align*}
%
\paragraph{(\ref{item:b1} $\boldsymbol{\Rightarrow}$ \ref{item:d1})}
Lemma~\ref{lem:data-dependent-coverage} implies that with probability at least $1-\delta/2$,
\begin{equation*}
\PDIS \VS_{\F,S_{m}}
\leq \frac{1}{m} \left( 10 \hatn{m} \ln\left( \frac{e m}{\hatn{m}} \right) + 4 \ln\left(\frac{4}{\delta}\right) \right),
\end{equation*}
while the definition of $\Bound{\hat{n}}{m}{\frac{\delta}{2}}$ implies that $\hatn{m} \leq \Bound{\hat{n}}{m}{\frac{\delta}{2}}$ with probability at least $1-\delta/2$.
By a union bound, both of these occur with probability at least $1-\delta$; together with the facts that $x \mapsto x \ln( e m / x )$ is nondecreasing on $(0,m]$ and $\Bound{\hat{n}}{m}{\frac{\delta}{2}} \leq m$,
this implies
\begin{align*}
\Bound{\PDIS}{m}{\delta}
& \leq \frac{1}{m} \left( 10 \Bound{\hat{n}}{m}{\frac{\delta}{2}} \ln\left( \frac{e m}{\Bound{\hat{n}}{m}{\frac{\delta}{2}}} \right) + 4 \ln\left(\frac{4}{\delta}\right) \right)
\\ & = O\left( \frac{1}{m} \left(\Bound{\hat{n}}{m}{\frac{\delta}{2}} \log(m) + \log\left(\frac{1}{\delta}\right) \right) \right).
\end{align*}
Thus, if $\Bound{\hat{n}}{m}{\delta} = O\left( \polylog( m ) \log\left(\frac{1}{\delta}\right)\right)$,
then we have
\begin{equation*}
\Bound{\PDIS}{m}{\delta}
= O\left( \frac{\polylog( m )}{m} \log\left(\frac{1}{\delta}\right) \right).
\end{equation*}
%
\paragraph{(\ref{item:d1} $\boldsymbol{\Rightarrow}$ \ref{item:d2})}
If $\Bound{\PDIS}{m}{\delta} = O\left( \frac{\polylog( m )}{m} \log\left(\frac{1}{\delta}\right) \right)$,
then there exists a sufficiently small constant $\delta_{4} \in (0, 1/9]$ such that
$\Bound{\PDIS}{m}{\delta_{4}} = O\left( \frac{\polylog(m)}{m} \right)$;
in fact, combined with monotonicity of $\delta \mapsto \Bound{\PDIS}{m}{\delta}$,
this implies $\Bound{\PDIS}{m}{\frac{1}{9}} = O\left( \frac{\polylog(m)}{m} \right)$ as well.
%
\paragraph{(\ref{item:d2} $\boldsymbol{\Rightarrow}$ \ref{item:c})}
If $\Bound{\PDIS}{m}{\frac{1}{9}} = O\left( \frac{\polylog(m)}{m} \right)$,
then Lemma~\ref{cor:boundingDiss} in Appendix~\ref{sec:dc-appendix} implies
\begin{align*}
\dc(r_{0}) & \leq \max\left\{\sup_{r \in (r_{0},1/2)} \frac{7 \Bound{\PDIS}{\lfloor 1/r \rfloor}{\frac{1}{9}}}{r}, 2 \right\}
\\ & \leq 2 + 14 \max_{m \leq 1/r_{0}} m \Bound{\PDIS}{m}{\frac{1}{9}}
\\ & = O\left( \max_{m \leq 1/r_{0}} \polylog(m) \right)
= O\left( \polylog\left(\frac{1}{r_{0}}\right) \right).
\end{align*}
%
%
%
\end{proof}

\section{Applications}
\label{sec:applications}

In this section, we state bounds on the complexity measures studied above,
for various hypothesis classes $\F$ and distributions $P$, which can then be used in 
conjunction with the above results.  In each case, combining the result with theorems above
yields a bound on the label complexity of \CAL~that is smaller than the best known result
in the published literature for that problem.

\subsection{Linear Separators under Mixtures of Gaussians}
\label{sec:linsep-gaussian}

The first result, due to \citet*{ElYaniv10}, applies to the problem of learning linear separators
under a mixture of Gaussians distribution.  Specifically, for $k \in \nats$, 
the class of linear separators in $\reals^{k}$ is defined as the set of classifiers
$(x_{1},\ldots,x_{k}) \mapsto \sign( b + \sum_{i=1}^{k} x_{i} w_{i} )$, where the values
$b,w_{1},\ldots,w_{k} \in \reals$ are free parameters specifying the classifier,
with $\sum_{i=1}^{k} w_{i}^{2} = 1$, and where $\sign(t) = 2\ind_{[0,\infty)}(t)-1$.  
In this work, we also include the two constant functions $x \mapsto -1$ and 
$x \mapsto +1$ as members of the class of linear separators.  

\begin{theorem}[\citealp*{ElYaniv10}, Lemma 32]
\label{thm:linsep-gaussian}
For $t,k \in \nats$, there is a finite constant $c_{k,t}$ $> 0$ 
such that, for $\F$ the space of linear separators on $\reals^{k}$,
and for $P$ with marginal distribution over $\cX$ that is a mixture of $t$ 
multivariate normal distributions with diagonal covariance matrices of full rank,
$\forall m \geq 2$,
\begin{equation*}
\Bound{\hat{n}}{m}{\frac{1}{20}} \leq c_{k,t} (\log(m))^{k-1}.
\end{equation*}
\end{theorem}

Combining this result with Theorem~\ref{thm:td-dc-bound} implies that there is a constant $c_{k,t} \in (0,\infty)$ 
such that, for $\F$ and $P$ as in Theorem~\ref{thm:linsep-gaussian}, $\forall r_{0} \in (0,1/2]$, 
\begin{equation*}
\dc(r_{0}) \leq c_{k,t} \left(\log\left(\frac{1}{r_{0}}\right)\right)^{k-1}.
\end{equation*}
In particular, plugging this into the label complexity bound of \citet*{hannekeAnnals11} for \CAL~(Lemma~\ref{lem:dc-LC-bound} of Appendix~\ref{sec:dc-appendix})
yields the following bound on the label complexity of \CAL, which has an improved asymptotic dependence on $\epsilon$
compared to the previous best known result, due to \citet*{el2012active}, reducing the 
exponent on the logarithmic factor from $\Theta(k^{2})$ to $\Theta(k)$, and reducing
the dependence on $\delta$ from $\poly(1/\delta)$ to $\log(1/\delta)$.

\begin{corollary}
\label{cor:linsep-gaussian-lc}
For $t,k \in \nats$, there is a finite constant $c_{k,t} > 0$ 
such that, for $\F$ the space of linear separators on $\reals^{k}$,
and for $P$ with marginal distribution over $\cX$ that is a mixture of $t$
multivariate normal distributions with diagonal covariance matrices of full rank,
$\forall \epsilon,\delta \in (0,1/2]$,
\begin{equation*}
\LC(\epsilon,\delta) \leq
c_{k,t} \left(\log\left( \frac{1}{\epsilon} \right)\right)^{k} \log\left( \frac{\log\left(1 / \epsilon\right)}{\delta} \right).
\end{equation*}
\end{corollary}

Corollary~\ref{cor:linsep-gaussian-lc} is particularly interesting in light of a lower bound of \citet*{el2012active} 
for this problem, showing that there exists a distribution $P$ of the type described in Corollary~\ref{cor:linsep-gaussian-lc}
for which $\Bound{N}{m}{\delta} = \Omega\left( \left( \log(m) \right)^{\frac{k-1}{2}} \right)$.

\subsection{Axis-aligned Rectangles under Product Densities}
\label{sec:rectangles}

The next result applies to the problem of learning axis-aligned rectangles under product densities over $\reals^{k}$:
that is, classifiers $h((x_{1}^{\prime},\ldots,x_{k}^{\prime})) = 2 \prod_{j=1}^{k} \ind_{[a_{j},b_{j}]}(x_{j}^{\prime}) - 1$,
for values $a_{1},\ldots,a_{k},b_{1},\ldots,b_{k} \in \reals$.  The result specifically applies to rectangles with 
a probability at least $\lambda > 0$ of classifying a random point positive.
This result represents a refinement of a result of \citet*{Hanneke_COLT_07}: specifically, reducing a factor of $k^{2}$ to a factor of $k$.

\begin{theorem}
\label{thm:rectangles}
For $k,m \in \nats$ and $\lambda,\delta \in (0,1)$, 
for any $P$ with marginal distribution over $\cX$ that is a product distribution
with marginals having continuous CDFs, and 
for $\F$ the space of axis-aligned rectangles $h$ on $\reals^{k}$ with $P((x,y) : h(x) = 1) \geq \lambda$,
\begin{equation*}
\Bound{\hat{n}}{m}{\delta} \leq \frac{8k}{\lambda} \ln\left( \frac{8k}{\delta} \right).
\end{equation*}
\end{theorem}
\begin{proof}
The proof is based on a slight refinement of an argument of \citet*{Hanneke_COLT_07}.
For $(X,Y) \sim P$, denote $(X_1,\ldots,X_k) \eqdef X$, 
let $G_{i}$ be the CDF of $X_{i}$, and define $G(X_1,\ldots,X_k) \eqdef (G_1(X_1),\ldots,G_k(X_k))$.  
Then the random variable $X^{\prime} \eqdef (X_{1}^{\prime},\ldots,X_{k}^{\prime}) \eqdef (G_{1}(X_{1}),\ldots,G_{k}(X_{k})) = G(X)$
is uniform in $(0,1)^{k}$; to see this, note that since $X_{1},\ldots,X_{k}$ are independent, so are $G_{1}(X_{1}),\ldots,G_{k}(X_{k})$,
and that for each $i \leq k$, $\forall t \in (0,1)$, $\P( G_{i}(X_{i}) \leq t ) = \sup_{x \in \reals : G_{i}(x) = t} \P( X_{i} \leq x) = \sup_{x \in \reals : G_{i}(x) = t} G_{i}( x ) = t$,
where the first equality is by monotonicity and continuity of $G_{i}$ and the intermediate value theorem (since $\lim_{x \to -\infty} G_{i}(x) = 0 < t$ and $\lim_{x\to\infty}G_{i}(x)=1 > t$), 
and the second equality is by definition of $G_{i}$. 
Fix any $h \in \F$, let $a_{1},\ldots,a_{k},b_{1},\ldots,b_{k} \in \reals$ be the values such that 
$h( (z_{1},\ldots,z_{k}) ) = 2 \prod_{i=1}^{k} \ind_{[a_{i},b_{i}]}(z_{i}) - 1$ for all $(z_{1},\ldots,z_{k}) \in \reals^{k}$,
and define $H_{h}((z_{1},\ldots,z_{k})) = 2 \prod_{i=1}^{k} \ind_{[ G_{i}(a_{i}), G_{i}(b_{i}) ]}(z_{i}) - 1$.
Clearly $H_{h}$ is an axis-aligned rectangle.  Furthermore, for every $z \in \reals^{k}$ with $h(z) = +1$,
monotonicity of the $G_{i}$ functions implies $H_{h}(G(z)) = +1$ as well.  Therefore, 
$\P( H_{h}(X^{\prime}) = +1) \geq \P( h(X) = +1 ) \geq \lambda$.

Let $G_{i}^{-1}(t) = \min\{ s : G_{i}(s) = t \}$ for $t \in (0,1)$, which is well-defined by continuity of $G_{i}$ and the intermediate value theorem,
combined with the facts that $\lim_{z\to\infty}G_{i}(z)=1$ and $\lim_{z\to -\infty}G_{i}(z)=0$.
Let $T_{i}$ denote the set of discontinuity points of $G_{i}^{-1}$ in $(0,1)$.
Fix any $(z_{1},\ldots,z_{k}) \in \reals^{k}$ with $h((z_{1},\ldots,z_{k})) = -1$ and $G(z_{1},\ldots,z_{k}) \in (0,1)^{k}$. 
In particular, this implies $\exists i \in \{1,\ldots,k\}$ such that $z_{i} \notin [a_{i},b_{i}]$.
For this $i$, we have $G_{i}(z_{i}) \notin (G_{i}(a_{i}),G_{i}(b_{i}))$ by monotonicity of $G_{i}$.
Therefore, if $H_{h}(G(z_{1},\ldots,z_{k}))=+1$, we must have either $z_{i} < a_{i}$ and $G_{i}(z_{i})=G_{i}(a_{i})$,
or $z_{i} > b_{i}$ and $G_{i}(z_{i}) = G_{i}(b_{i})$.  In the former case, 
for any $\epsilon$ with $0 < \epsilon < 1-G_{i}(z_{i})$, 
$G_{i}^{-1}(G_{i}(z_{i})+\epsilon) = G_{i}^{-1}(G_{i}(a_{i})+\epsilon) > a_{i}$,
while $G_{i}^{-1}(G_{i}(z_{i})) \leq z_{i}$, and since $z_{i} < a_{i}$, we must have $G_{i}(z_{i}) \in T_{i}$.
Similarly, in the latter case ($z_{i} > b_{i}$ and $G_{i}(z_{i}) = G_{i}(b_{i})$), 
any $\epsilon$ with $0 < \epsilon < 1-G_{i}(z_{i})$ has $G_{i}^{-1}(G_{i}(b_{i})+\epsilon) = G_{i}^{-1}(G_{i}(z_{i})+\epsilon) > z_{i}$,
while $G_{i}^{-1}(G_{i}(b_{i})) \leq b_{i}$, and since $z_{i} > b_{i}$, we have $G_{i}(b_{i}) \in T_{i}$; since $G_{i}(z_{i}) = G_{i}(b_{i})$, 
this also implies $G_{i}(z_{i}) \in T_{i}$.
Thus, any $(z_{1},\ldots,z_{k}) \in \reals^{k}$ with $H_{h}(G(z_{1},\ldots,z_{k})) \neq h((z_{1},\ldots,z_{k}))$
must have some $i \in \{1,\ldots,k\}$ with $G_{i}(z_{i}) \in T_{i}$.

For each $i \in \{1,\ldots,k\}$, since $G_{i}$ is nondecreasing, $G_{i}^{-1}$ is also nondecreasing,
and this implies $G_{i}^{-1}$ has at most countably many discontinuity points \citep*[see e.g.,][Section 31, Theorem 1]{kolmogorov:75}.
Furthermore, for every $t \in \reals$, 
\begin{align*}
\P(G_{i}(X_{i})=t) 
& \leq \P\left( \inf\{ x \in \reals : G_{i}(x) = t \} \leq X_{i} \leq \sup\{x \in \reals : G_{i}(x) = t\} \right) 
\\ & = G_{i}( \sup\{x \in \reals : G_{i}(x) = t\}) - G_{i}( \inf\{x \in \reals : G_{i}(x) = t\} ) = t - t = 0,
\end{align*}
where the inequality is due to monotonicity of $G_{i}$, the first equality is 
by definition of $G_{i}$ as the CDF and by continuity of $G_{i}$ (which implies $\P(X_{i} < x)=G_{i}(x)$),
and the second equality is due to continuity of $G_{i}$.
Therefore, 
\begin{equation*}
\P\left( \exists h \in \F : H_{h}(G(X)) \neq h(X) \right) 
\leq \P\left( \exists i \in \{1,\ldots,k\} : G_{i}(X_{i}) \in T_{i} \right)
\leq \sum_{i=1}^{k} \sum_{t \in T_{i}} \P( G_{i}(X_{i}) = t ) = 0.
\end{equation*}
By a union bound, this implies that with probability $1$, for every $h \in \F$, every $(x,y) \in S_{m}$ has 
$H_{h}(G(x)) = h(x)$.
In particular, we have that with probability $1$, every classification of the sequence $\{x_1,\ldots,x_m\}$ realized by 
classifiers in $\F$ is also realized as a classification of the i.i.d. Uniform($(0,1)^{k}$) sequence
$\{G(x_1),\ldots,G(x_m)\}$ by the set $\F^{\prime}$ of axis-aligned rectangles $h^{\prime}$ with $\P( h^{\prime}(X^{\prime}) = +1) \geq \lambda$.
This implies that $\Bound{\hat{n}}{m}{\delta} \leq \min\{ b \in \nats \cup \{0\} : \P( \hat{n}(\F^{\prime},\{(G(x),y) : (x,y) \in S_{m}\}) \leq b) \geq 1-\delta \}$
(in fact, one can show they are equal).  Therefore, since the right hand side is the value of $\Bound{\hat{n}}{m}{\delta}$ 
one would get from the case of $P$ having marginal $P(\cdot \times \cY)$ over $\cX$ that is Uniform($(0,1)^{k}$), 
without loss of generality, 
it suffices to bound $\Bound{\hat{n}}{m}{\delta}$ for this special case.  Toward this end, for the remainder
of this proof, we assume $P$ has marginal $P(\cdot \times \cY)$ over $\cX$ uniform in $(0,1)^{k}$.

Let $m \in \nats$, and let $\U = \{ x_{1},\ldots, x_{m} \}$, the unlabeled portion of the first $m$ data points.
Further denote by $\U^{+} = \{x_{i} \in \U : \target(x_{i}) = +1\}$, and $\U^{-} = \U \setminus \U^{+}$.
For each $i \in \nats$, express $x_{i}$ explicitly in vector form as $(x_{i1},\ldots,x_{ik})$.
If $\U^{+} \neq \emptyset$, for each $j \in \{1,\ldots,k\}$, let $a_{j} = \min\{ x_{ij} : x_{i} \in \U^{+} \}$ and $b_{j} = \max\{ x_{ij} : x_{i} \in \U^{+} \}$.
Denote by $h_{{\rm clos}}(x) = 2 \ind_{\times_{j=1}^{k} [a_{j},b_{j}]}(x) - 1$, the \emph{closure} hypothesis;
for completeness, when $\U^{+} = \emptyset$, let $h_{{\rm clos}}(x) = -1$ for all $x$.

First, note that if $m < \frac{2e}{\lambda} \left( 2k + \ln\left( \frac{2}{\delta} \right) \right)$, the result
trivially holds, since $\hatn{m} \leq m$ always, and $\frac{2e}{\lambda} \left( 2k + \ln\left( \frac{2}{\delta} \right) \right) \leq \frac{8k}{\lambda} \ln\left( \frac{8k}{\delta} \right)$. 
Otherwise, if $m \geq \frac{2e}{\lambda} \left( 2k + \ln\left( \frac{2}{\delta} \right) \right)$,
a result of \citet*{auer:04} implies that, 
on an event $E_{{\rm clos}}$ of probability at least $1-\delta/2$,
$P((x,y) : h_{{\rm clos}}(x) \neq \target(x)) \leq \lambda/2$.
In particular, since $P((x,y) : \target(x) = +1) \geq \lambda$,
on this event we must have $P((x,y) : h_{{\rm clos}}(x) = +1) \geq \lambda/2$.
Furthermore, this implies $\U^{+} \neq \emptyset$ on $E_{{\rm clos}}$.

Now fix any $j \in \{1,\ldots,k\}$.  Let $x^{(aj)}_{j}$ denote the value $x_{ij}$ for the point $x_{i} \in \U$ with 
largest $x_{ij}$ such that $x_{ij} < a_{j}$, and for all $j^{\prime} \neq j$, $x_{i j^{\prime}} \in [a_{j^{\prime}},b_{j^{\prime}}]$;
if no such point exists, let $x^{(aj)}_{j} = 0$.
Let $\U^{(aj)} = \{x_{i} \in \U : x_{ij} < a_{j}\}$.
Let $m^{(aj)} = |\U^{(aj)}|$, and enumerate the points in $\U^{(aj)}$ in decreasing order of $x_{ij}$, 
so that $i_{1},\ldots,i_{m^{(aj)}}$ are distinct indices such that each $t \in \{1,\ldots,m^{(aj)}\}$ has $x_{i_{t}} \in \U^{(aj)}$,
and each $t \in \{1,\ldots,m^{(aj)}-1\}$ has $x_{i_{t+1}j} \leq x_{i_{t}j}$.
Since $P((x,y) : h_{{\rm clos}}(x) = +1) \geq \lambda/2$ on $E_{{\rm clos}}$, it must be that the volume of $\times_{j^{\prime} \neq j} [a_{j^{\prime}},b_{j^{\prime}}]$
is at least $\lambda/2$.  Therefore, working under the conditional distribution given $\U^{+}$ and $m^{(aj)}$, on $E_{{\rm clos}}$,
for each $t \in \{1,\ldots,m^{(aj)}\}$, with conditional probability at least $\lambda/2$, 
we have $\forall j^{\prime} \neq j$, $x_{i_{t}j^{\prime}} \in [a_{j^{\prime}},b_{j^{\prime}}]$.  Therefore, 
the value $t^{(aj)} \eqdef \min\{ t : \forall j^{\prime} \neq j, x_{i_{t}j^{\prime}} \in [a_{j^{\prime}},b_{j^{\prime}}]\} \cup \{m^{(aj)}\}$
is bounded by a Geometric random variable with parameter $\lambda/2$.
In particular, this implies that with conditional probability at least $1-\frac{\delta}{4k}$, 
$t^{(aj)} \leq \left\lceil \frac{2}{\lambda} \ln\left( \frac{4k}{\delta} \right) \right\rceil$.
Letting $A^{(aj)} = \{ x_{i} \in \U : x^{(aj)}_{j} \leq x_{ij} < a_{j} \}$, we note that 
$|A^{(aj)}| \leq t^{(aj)}$ with probability $1$, 
so that the above reasoning, combined with the law of total probability, implies that 
there is an event $E^{(aj)}$ of probability at least $1-\frac{\delta}{4k}$ such that, 
on $E^{(aj)} \cap E_{{\rm clos}}$, 
$|A^{(aj)}| \leq \left\lceil \frac{2}{\lambda} \ln\left( \frac{4k}{\delta} \right) \right\rceil$.
For the symmetric case, define $x^{(bj)}_{j}$ as the value $x_{ij}$ for the point $x_{i} \in \U$ with 
smallest $x_{ij}$ such that $x_{ij} > b_{j}$, and for all $j^{\prime} \neq j$, $x_{i j^{\prime}} \in [a_{j^{\prime}},b_{j^{\prime}}]$;
if no such point $x_{i}$ exists, define $x^{(bj)}_{j} = 1$.
Define $A^{(bj)} = \{x_{i} \in \U : b_{j} < x_{ij} \leq x^{(bj)}_{j} \}$.
By the same reasoning as above, there is an event $E^{(bj)}$ of probability at least $1-\frac{\delta}{4k}$
such that, on $E^{(bj)} \cap E_{{\rm clos}}$, $|A^{(bj)}| \leq \left\lceil \frac{2}{\lambda} \ln\left( \frac{4k}{\delta} \right) \right\rceil$.
Applying this to all values of $j$, and letting $A = \bigcup_{j=1}^{k} A^{(aj)} \cup A^{(bj)}$,
we have that on the event $E_{{\rm clos}} \cap \bigcap_{j=1}^{k} E^{(aj)} \cap E^{(bj)}$, 
\begin{equation*}
|A| \leq 2k \left\lceil \frac{2}{\lambda} \ln\left( \frac{4k}{\delta} \right) \right\rceil.
\end{equation*}
Furthermore, a union bound implies that the event $E_{{\rm clos}} \cap \bigcap_{j=1}^{k} E^{(aj)} \cap E^{(bj)}$ 
has probability at least $1-\delta$.  For the remainder of the proof, we suppose this event occurs.

Next, let $B = \left\{ \argmin\limits_{x_{i} \in \U^{+}} x_{ij} : j \in \{1,\ldots,k\} \right\} \cup \left\{ \argmax\limits_{x_{i} \in \U^{+}} x_{ij} : j \in \{1,\ldots,k\} \right\}$,
and note that $|B| \leq 2k$.
Finally, we conclude the proof by showing that the set $A \cup B$ has the property that
$\{h \in \F : \forall x \in A \cup B, h(x) = \target(x)\} = \VS_{\F,S_{m}}$, which implies $\{(x_{i},y_{i}) : x_{i} \in A \cup B\}$
is a version space compression set, so that $\hatn{m} \leq |A \cup B|$, and hence 
$\Bound{\hat{n}}{m}{\delta} \leq 2k + 2k \left\lceil \frac{2}{\lambda} \ln\left( \frac{4k}{\delta} \right) \right\rceil 
\leq \frac{8k}{\lambda} \ln\left( \frac{4k}{\delta} \right)$. 
To prove that $A \cup B$ has this property, first note that any $h \in \F$ with $h(x_{i}) = +1$ for all $x_{i} \in B$, 
must have $\U^{+} \supseteq \{x_{i} \in \U^{+} : h(x_{i}) = +1\} \supseteq \U^{+} \cap \times_{j=1}^{k} [ \min_{x_{i} \in \U^{+}} x_{ij}, \max_{x_{i} \in \U^{+}} x_{ij} ] = \U^{+}$,
so that
$\{x_{i} \in \U : h(x_{i}) = +1\} \supseteq \U^{+} = \{x_{i} \in \U : \target(x_{i}) = +1\}$.
Next, for any $x_{i} \in \U^{-} \setminus (A \cup B)$, $\exists j \in \{1,\ldots,k\} : x_{ij} \notin [a_{j},b_{j}]$,
and by definition of $A$, for this $j$ we must have $x_{ij} \notin [ x^{(aj)}_{j}, x^{(bj)}_{j} ]$.
Now fix any $h \in \F$, and express $\{x : h(x)=+1\} = \times_{j^{\prime}=1}^{k} [a_{j^{\prime}}^{\prime},b_{j^{\prime}}^{\prime}]$.
If $h(x_{i^{\prime}}) = +1$ for all $x_{i^{\prime}} \in B$, then we must have $a_{j^{\prime}}^{\prime} \leq a_{j^{\prime}}$ and $b_{j^{\prime}}^{\prime} \geq b_{j^{\prime}}$
for every $j^{\prime} \in \{1,\ldots,k\}$.  Furthermore, if $h(x_{i}) = +1$, then we must have $a_{j}^{\prime} \leq x_{ij} \leq b_{j}^{\prime}$;
but then we must have either $a_{j}^{\prime} \leq x_{ij} < x^{(aj)}_{j}$ or $x^{(bj)}_{j} < x_{ij} \leq b_{j}^{\prime}$.
In the former case, since $x_{ij} < x^{(aj)}_{j}$, we must have $x^{(aj)}_{j} > 0$, so that there exists a point $x_{i^{\prime}} \in \U$
with $x_{i^{\prime}j} = x^{(aj)}_{j}$ and with $x_{i^{\prime}j^{\prime}} \in [a_{j^{\prime}},b_{j^{\prime}}]$ for all $j^{\prime} \neq j$,
and furthermore (by definition of $A$), $x_{i^{\prime}} \in A$; but since $[a_{j^{\prime}},b_{j^{\prime}}] \subseteq [a_{j^{\prime}}^{\prime},b_{j^{\prime}}^{\prime}]$ 
we also have $x_{i^{\prime}j^{\prime}} \in [a_{j^{\prime}}^{\prime},b_{j^{\prime}}^{\prime}]$ for all $j^{\prime} \neq j$, and since $a_{j}^{\prime} < x^{(aj)}_{j} = x_{i^{\prime}j} < a_{j} \leq b_{j} \leq b_{j}^{\prime}$,
we also have $x_{i^{\prime}j} \in [a_{j}^{\prime},b_{j}^{\prime}]$.  Altogether, we must have $h(x_{i^{\prime}}) = +1$, 
which proves there exists at least one point in $A \cup B$ classified differently by $h$ and $\target$.
The case that $x^{(bj)}_{j} < x_{ij} \leq b_{j}^{\prime}$ is symmetric to this one, so that by the same reasoning, 
this $h$ must disagree with $\target$ on the classification of some point in $A \cup B$.
Therefore, every $h \in \F$ with $h(x) = \target(x)$ for all $x \in A \cup B$ has $h(x_{i}) = -1$ for all $x_{i} \in \U^{-} \setminus (A \cup B)$.
Combined with the above proof that every such $h$ also has $h(x_{i}) = +1$ for every $x_{i} \in \U^{+}$, 
we have that every such $h$ has $h(x) = \target(x)$ for every $x \in \U$.
\end{proof}

One implication of Theorem~\ref{thm:rectangles}, combined with Theorem~\ref{thm:td-dc-bound}, is that 
\begin{equation*}
\dc(r_{0}) \leq 128 \frac{k}{\lambda} \ln(160 k)
\end{equation*}
for all $r_{0} \geq 0$, for $P$ and $\F$ as in Theorem~\ref{thm:rectangles}.  This has implications, both for 
the label complexity of \CAL~(via Lemma~\ref{lem:dc-LC-bound}), and also for the label complexity of 
noise-robust disagreement-based methods (see Section~\ref{sec:noise} below).
More directly, combining Theorem~\ref{thm:rectangles} with Theorem~\ref{thm:cal-label-complexity}
yields the following label complexity bound for \CAL, which improves over the best previously published bound on the label complexity of 
\CAL~for this problem (due to \citealp*{el2012active}), reducing the dependence on $k$ from $\Theta(k^{3} \log^{2}(k) )$ to $\Theta(k \log^{2}(k))$.

\begin{corollary}
\label{cor:rectangles-lc}
There exists a finite universal constant $c > 0$ such that,
for $k \in \nats$ and $\lambda \in (0,1)$, 
for any $P$ with marginal distribution over $\cX$ that is a product distribution
with marginals having continuous CDFs, and 
for $\F$ the space of axis-aligned rectangles $h$ on $\reals^{k}$ with $P((x,y) : h(x) = 1) \geq \lambda$,
$\forall \epsilon,\delta \in (0,1/2)$, 
\begin{equation*}
\LC(\epsilon,\delta)
\leq c \frac{k}{\lambda} \log\left(\frac{k}{\delta}\log\left(\frac{1}{\epsilon}\right)\right) \log\left(\frac{k}{\epsilon}\log\left(\frac{1}{\delta}\right)\right) \log\left(\frac{\lambda \log(1/\epsilon)}{\epsilon \log(k)} \lor e\right).
\end{equation*}
\end{corollary}
\begin{proof}
The result follows by plugging the bound from Theorem~\ref{thm:rectangles} into Theorem~\ref{thm:cal-label-complexity},
taking $\delta_{m} = \delta/(2\log_{2}(2M(\epsilon,\delta/2)))$,
bounding $M(\epsilon,\delta/2) \leq \frac{8 k}{\epsilon}\log(\frac{8e}{\epsilon})+\frac{8}{\epsilon}\log(\frac{24}{\delta})$ \citep*{vapnik:82,AnthoB99},  
and simplifying the resulting expression.
\end{proof}

This result is particularly interesting in light of the following lower bound on the label complexities 
achievable by \emph{any} active learning algorithm.

\begin{theorem}
\label{thm:rectangles-lower-bound}
For $k \in \nats \setminus \{1\}$ and $\lambda \in (0,1/4]$, 
letting $P_{X}$ denote the uniform probability distribution over $(0,1)^{k}$,
for $\F$ the space of axis-aligned rectangles $h$ on $\reals^{k}$ with $P_{X}(x : h(x) = 1) \geq \lambda$,
for any active learning algorithm $\mathcal{A}$, $\forall \delta \in (0,1/2]$, $\forall \epsilon \in (0,1/(8k))$,
there exists a function $\target \in \F$ such that, if $P$ is the realizable-case distribution having
marginal $P_{X}$ over $\cX$ and having target function $\target$,
if $\mathcal{A}$ is allowed fewer than 
\begin{equation*}
\max\left\{ k \log\left(\frac{1}{4 k \epsilon}\right), (1-\delta) \left\lfloor \frac{1}{\epsilon \lor \lambda} \right\rfloor \right\} - 1
\end{equation*}
label requests, then with probability greater than $\delta$, the returned classifier $\hat{h}$ has $\er(\hat{h}) > \epsilon$.
\end{theorem}
\begin{proof}
For any $\epsilon > 0$, let $\mathcal{M}(\epsilon)$ denote the maximum number $M$ of classifiers $h_1,\ldots,h_M \in \F$ such that,
$\forall i,j \leq M$ with $i \neq j$, $P_{X}(x : h_{i}(x) \neq h_{j}(x)) \geq 2\epsilon$.
\citet*{kulkarni:93} prove that, for any learning algorithm based on binary-valued queries, 
with a budget smaller than $\log_{2}( (1-\delta)\mathcal{M}(2\epsilon) )$ queries, there 
exists a target function $\target \in \F$ such that the classifier $\hat{h}$ produced by the 
algorithm (when $P$ has marginal $P_{X}$ over $\cX$ and has target function $\target$)
will have $\er(\hat{h}) > \epsilon$ with probability greater than $\delta$.
In particular, since active learning queries are binary-valued in the binary classification setting,
this lower bound applies to active learning algorithms as a special case.

Thus, for the first term in the lower bound, we focus on establishing a lower bound on $\mathcal{M}(2\epsilon)$ for this problem.
First note that $(1-1/k)^{k} \geq 1/4$, so that $\lambda \leq (1-1/k)^{k}$.
Furthermore, $(1/k) (1-1/k)^{k-1} > 1/(4k)$, so that $\epsilon < (1/k)(1-1/k)^{k-1}$.
Now let 
\begin{multline*}
\F_{2\epsilon} = \left\{ (x_{1},\ldots,x_{k}) \mapsto 2 \prod_{j=1}^{k} \ind_{[a_{j},b_{j}]}(x_{j}) - 1 : \forall j \leq k, b_{j} = a_{j} + 1-1/k, \right.
\\ \left. \phantom{\prod_{j=1}^{k}} a_{j} \in \left\{0, \frac{\epsilon}{(1-1/k)^{k-1}}, \ldots, \left\lfloor \frac{(1-1/k)^{k-1}}{\epsilon k} \right\rfloor \frac{\epsilon}{(1-1/k)^{k-1}} \right\} \right\}.
\end{multline*}
Note that $|\F_{2\epsilon}| = \left(1 + \left\lfloor \frac{(1-1/k)^{k-1}}{\epsilon k} \right\rfloor \right)^{k}$.
Furthermore, since every $a_{j} \in [0,1/k]$ in the specification of $\F_{2\epsilon}$, we have $b_{j} = a_{j} + 1-1/k \in [0,1]$, which implies 
$P_{X}( (x_{1},\ldots,x_{k}) : \prod_{j=1}^{k} \ind_{[a_{j},b_{j}]}(x_{j}) = 1 ) = (1-1/k)^{k} \geq \lambda$.  Therefore, $\F_{2\epsilon} \subseteq \F$.
Finally, for each $\{(a_{j},b_{j})\}_{j=1}^{k}$ and $\{(a_{j}^{\prime},b_{j}^{\prime})\}_{j=1}^{k}$ specifying distinct classifiers in $\F_{2\epsilon}$,
at least one $j$ has $|a_{j} - a_{j}^{\prime}| \geq \frac{\epsilon}{(1-1/k)^{k-1}}$.
Since all of the elements $h \in \F_{2\epsilon}$ have $P_{X}(x : h(x) = +1) = (1-1/k)^{k}$, 
we can note that 
\begin{align*}
& P_{X}\left( (x_{1},\ldots,x_{k}) : \prod_{i=1}^{k} \ind_{[a_{i},b_{i}]}(x_{i}) \neq \prod_{i=1}^{k} \ind_{[a_{i}^{\prime},b_{i}^{\prime}]}(x_{i}) \right) 
\\ & = 2(1-1/k)^{k} - 2P_{X}\left( (\times_{i=1}^{k} [a_{i},b_{i}]) \cap (\times_{i=1}^{k} [a_{i}^{\prime},b_{i}^{\prime}]) \right)
\\ & = 2(1-1/k)^{k} - 2P_{X}\left( \times_{i=1}^{k} [\max\{a_{i},a_{i}^{\prime}\},\min\{b_{i},b_{i}^{\prime}\}] \right)
\\ & = 2(1-1/k)^{k} - 2 \prod_{i=1}^{k} (\min\{b_{i},b_{i}^{\prime}\} - \max\{a_{i},a_{i}^{\prime}\}).
\end{align*}
Thus, since 
\begin{align*}
& \prod_{i=1}^{k} (\min\{b_{i},b_{i}^{\prime}\} - \max\{a_{i},a_{i}^{\prime}\}) 
\\ & \leq (\min\{b_{j},b_{j}^{\prime}\}-\max\{a_{j},a_{j}^{\prime}\}) \prod_{i \neq j} (b_{i}-a_{i}) 
= (1-1/k)^{k-1} (\min\{b_{j},b_{j}^{\prime}\} - \max\{a_{j},a_{j}^{\prime}\} ) 
\\ & = (1-1/k)^{k-1} ( \min\{ a_{j},a_{j}^{\prime} \} - \max\{a_{j},a_{j}^{\prime}\} + (1-1/k) )
= (1-1/k)^{k-1} ( 1 - 1/k - |a_{j}-a_{j}^{\prime}| ) 
\\ & \leq (1-1/k)^{k-1} (1 - 1/k - \frac{\epsilon}{(1-1/k)^{k-1}}) = (1-1/k)^{k} - \epsilon,
\end{align*}
we have 
\begin{equation*}
P_{X}( (x_{1},\ldots,x_{k}) : \prod_{i=1}^{k} \ind_{[a_{i},b_{i}]}(x_{i}) \neq \prod_{i=1}^{k} \ind_{[a_{i}^{\prime},b_{i}^{\prime}]}(x_{i}) ) 
\geq 2 (1-1/k)^{k} - 2 ( (1-1/k)^{k} - \epsilon ) = 2\epsilon.
\end{equation*}
Thus, $\mathcal{M}(2\epsilon) \geq \left(1 + \left\lfloor \frac{(1-1/k)^{k-1}}{\epsilon k} \right\rfloor \right)^{k}$.
Finally, note that for $\delta \in (0,1/2]$, this implies 
\begin{equation*}
\log_{2}( (1-\delta) \mathcal{M}(2\epsilon) ) 
\geq k \log_{2}\left( \frac{(1-1/k)^{k-1}}{\epsilon k} \right) - 1 
\geq k \log_{2}\left( \frac{1}{4 k \epsilon} \right) - 1.
\end{equation*}
Together with the aforementioned lower bound of \citet*{kulkarni:93}, this establishes the first term in the lower bound.

To prove the second term, we use of a technique of \citet*{Hanneke_COLT_07}.
Specifically, fix any finite set $H \subseteq \F$
with $\min_{h,g \in H} P_{X}(x : h(x) \neq g(x)) \geq 2\epsilon$, 
let 
\begin{equation*}
{\rm XPTD}(f,H,\U,\delta) = \min\{ t \in \nats : \exists R \subseteq \U : |R| \leq t, |\{ h \in H : \forall x \in R, h(x) = f(x)\}| \leq \delta |H| + 1 \} \!\cup\! \{\infty\},
\end{equation*}
for any classifier $f$ and $\U \in \bigcup_{m}\cX^{m}$,
and let ${\rm XPTD}(H,P_{X},\delta)$ denote the smallest $t \in \nats$ such that
every classifier $f$ has $\lim_{m \to \infty} \P_{\U \sim P_{X}^{m}}\left( {\rm XPTD}(f,H,\U,\delta) > t \right) = 0$.
Then \citet*{Hanneke_COLT_07} proves that there exists a choice of target function $\target \in \F$ for 
the distribution $P$ such that, if $\mathcal{A}$ is allowed fewer than ${\rm XPTD}(H,P_{X},\delta)$ label requests,
then with probability greater than $\delta$, the returned classifier $\hat{h}$ has $\er(\hat{h}) > \epsilon$.
For the particular problem studied here, let $H$ be the set of classifiers $h_{i}(x) = 2 \ind_{[(i-1) (\epsilon \lor \lambda), i (\epsilon \lor \lambda)] \times [0,1]^{k-1}}(x) - 1$,
for $i \in \left\{1,\ldots, \left\lfloor \frac{1}{\epsilon \lor \lambda} \right\rfloor \right\}$.  Note that each $h_{i} \in H$ has 
$P_{X}(x : h_{i}(x) = +1) = P_{X}( (x_1,\ldots,x_k) : x_{1} \in [ (i-1) (\epsilon \lor \lambda), i (\epsilon \lor \lambda) ] ) = \epsilon \lor \lambda \geq \lambda$,
so that $H \subseteq \F$.
Furthermore, for any $h_{i}, h_{j} \in H$ with $i \neq j$, 
$P_{X}(x : h_{i}(x) \neq h_{j}(x)) \geq P_{X}( (x_1,\ldots,x_k) : x_{1} \in ( (i-1) (\epsilon \lor \lambda), i (\epsilon \lor \lambda) ) \cup ( (j-1) (\epsilon \lor \lambda), j (\epsilon \lor \lambda) ) ) = 2 (\epsilon \lor \lambda) \geq 2 \epsilon$.
Also, let $R \subseteq (0,1)^{k}$ be any finite set with no points $(x_{1},\ldots,x_{k}) \in R$ such that 
$x_{1} \in \left\{ i (\epsilon \lor \lambda) : i \in \left\{1,\ldots,\left\lfloor \frac{1}{\epsilon \lor \lambda} \right\rfloor - 1 \right\}\right\}$;
note that every $x \in R$ has exactly one $h_{i} \in H$ with $h_{i}(x) = +1$.  Thus, for the classifier $f$ with $f(x) = -1$ for all $x \in \cX$,
$|\{ h \in H : \forall x \in R, h(x) = f(x) \}| \geq |H|-|R|$.  Thus, for any set $\U \subseteq (0,1)^{k}$ with no points 
$(x_{1},\ldots,x_{k}) \in \U$ having 
$x_{1} \in \left\{ i (\epsilon \lor \lambda) : i \in \left\{1,\ldots,\left\lfloor \frac{1}{\epsilon \lor \lambda} \right\rfloor - 1 \right\}\right\}$,
we have ${\rm XPTD}(f,H,\U,\delta) \geq (1-\delta) |H| - 1$.
Since, for all $m \in \nats$, the probability that $\U \sim P_{X}^{m}$ contains a point $(x_1,\ldots,x_k)$
with $x_{1} \in \left\{ i (\epsilon \lor \lambda) : i \in \left\{1,\ldots,\left\lfloor \frac{1}{\epsilon \lor \lambda} \right\rfloor - 1 \right\}\right\}$
is zero, we have that 
$\P_{\U \sim P_{X}^{m}}( {\rm XPTD}(f,H,\U,\delta) \geq (1-\delta) |H| - 1 ) = 1$.
This implies ${\rm XPTD}(H,P_{X},\delta) \geq (1-\delta) |H| - 1 = (1-\delta) \left\lfloor \frac{1}{\epsilon \lor \lambda} \right\rfloor - 1$.
Combining this with the lower bound of \citet*{Hanneke_COLT_07} implies the result.
\end{proof}

Together, Corollary~\ref{cor:rectangles-lc} and Theorem~\ref{thm:rectangles-lower-bound} imply that,
for $\lambda \in (0,1/4]$ bounded away from $0$, the label complexity of \CAL~is within 
logarithmic factors of the minimax optimal label complexity.

\section{New Label Complexity Bounds for Agnostic Active Learning}
\label{sec:noise}

In this section we present new bounds on the label complexity of noise-robust active learning algorithms,
expressed in terms of $\Bound{\hat{n}}{m}{\delta}$.  These bounds yield new
exponential label complexity speedup results for agnostic active learning (for the low accuracy regime)
of linear classifiers under a fixed mixture of Gaussians.  Analogous results also hold for the problem
of learning axis-aligned rectangles under a product density.

Specifically, in the \emph{agnostic} setting studied in this section, we no longer
assume $\exists \target \in \F$ with $\P(Y=\target(x)|X)=1$ for $(X,Y) \sim P$,
but rather allow that $P$ is \emph{any} probability measure over $\cX \times \cY$.
In this setting, we let $\target : \cX \to \cY$ denote a classifier such that
$\er(\target) = \inf_{h \in \F} \er(h)$ and $\inf_{h \in \F} P( (x,y) : h(x) \neq \target(x)) = 0$, 
which is guaranteed to exist by topological considerations \citep*[see][Section 6.1]{Hanneke11};\footnote{In the agnostic setting, 
there are typically many valid choices of the function $\target$ satisfying
these conditions.  The results below hold for \emph{any} such choice of $\target$.}
for simplicity, when $\exists f \in \F$ with $\er(f) = \inf_{h \in \F} \er(h)$, we take $\target$ to be an element of $\F$.
We call $\target$ the \emph{infimal} hypothesis (of $\F$, w.r.t. $P$)
and note that $\er(\target)$ is sometimes called the \emph{noise rate of} $\F$ \citep*[e.g.,][]{balcan2006agnostic}.
The introduction of the infimal hypothesis $\target$ allows for natural 
generalizations of some of the key definitions of Section~\ref{sec:definitions} that facilitate 
analysis in the agnostic setting.

\begin{definition}[Agnostic Version Space]
Let $\target$ be the infimal hypothesis of $\F$ w.r.t. $P$. The \emph{agnostic version space} of a sample $S$ is
\begin{equation*}
\VS_{\F,S,\target} \eqdef \{ h \in \F : \forall (x,y) \in S, h(x) = \target(x) \}.
\end{equation*}
\end{definition}

\begin{definition}[Agnostic Version Space Compression Set Size]
~Letting $\VScomp_{S,\target}$ denote a smallest subset of $S$ satisfying $\VS_{\F,\VScomp_{S,\target},\target} = \VS_{\F,S,\target}$,
the \emph{agnostic version space compression set size} is 
\begin{equation*}
\hat{n}(\F,S,\target) \eqdef |\VScomp_{S,\target}|.
\end{equation*}
\end{definition}
We also extend the definition of the version space compression set minimal \emph{bound} (see (\ref{eq:vscompBound})) 
to the agnostic setting, defining
\begin{equation*}
\Bound{\hat{n}}{m}{\delta} \eqdef \min\{ b \in \nats \cup \{0\} : \P( \hat{n}(\F,S,\target) \leq b ) \geq 1-\delta \}.
\end{equation*}

For general $P$ in the agnostic setting, define the disagreement coefficient 
as before, except now with respect to the infimal hypothesis:
\begin{equation*}
\dc(r_{0}) \eqdef \sup_{r > r_{0}} \frac{\PDIS \Ball(\target,r)}{r} \lor 1 .
\end{equation*} 

One can easily verify that these definitions are equal to those given above in the special case that $P$ 
satisfies the realizable-case assumptions ($\target \in \F$ and $\P(Y=\target(X)|X)=1$ for $(X,Y) \sim P$).
%

We begin with the following extension of Theorem~\ref{thm:td-dc-bound}.

\begin{lemma}
\label{lem:td-dc-bound-agnostic}
For general (agnostic) $P$, for any $r_{0} \in (0,1)$,
\begin{equation*}
\dc(r_{0}) \leq \max\left\{ \max_{r \in (r_{0},1)} 16 \Bound{\hat{n}}{\left\lceil \frac{1}{r} \right\rceil}{\frac{1}{20}}, 512 \right\}.
\end{equation*}
\end{lemma}
\begin{proof}
First note that $\dc(r_{0})$ and $\Bound{\hat{n}}{\left\lceil \frac{1}{r} \right\rceil}{\frac{1}{20}}$ 
depend on $P$ only via $\target$ and the marginal $P(\cdot \times \cY)$ of $P$ over $\cX$ (in both the realizable case and agnostic case).
Define a distribution $P^{\prime}$ with marginal $P^{\prime}(\cdot \times \cY) = P(\cdot \times \cY)$ over $\cX$,
and with $\P(Y=\target(x)|X=x)=1$ for all $x \in \cX$, where $(X,Y) \sim P^{\prime}$.
In particular, in the special case that $\target \in \F$ in the agnostic case, 
we have that $P^{\prime}$ is a distribution in the realizable case, 
with identical values of $\dc(r_{0})$ and $\Bound{\hat{n}}{\left\lceil \frac{1}{r} \right\rceil}{\frac{1}{20}}$
as $P$, so that Theorem~\ref{thm:td-dc-bound} (applied to $P^{\prime}$) implies the result.
On the other hand, when $P$ is a distribution with $\target \notin \F$, let $\dc^{\prime}(r_{0})$
denote the disagreement coefficient of $\F \cup \{\target\}$ with respect to $P^{\prime}$ (or equivalently $P$),
and for $m \in \nats$, let ${\cal{B}}^{\prime}_{\hat{n}}(m,1/20) \eqdef \min\left\{ b \in \nats \cup \{0\} : \P(\hat{n}(\F\cup\{\target\},S_{m},\target) \leq b) \geq 19/20 \right\}$.
In particular, since $\F \subseteq \F \cup \{\target\}$, 
we have $\dc(r_{0}) \leq \dc^{\prime}(r_{0})$, 
and since $P^{\prime}$ is a realizable-case distribution with respect to the hypothesis class $\F \cup \{\target\}$,
Theorem~\ref{thm:td-dc-bound} (applied to $P^{\prime}$ and $\F \cup \{\target\}$) implies 
$$
\dc^{\prime}(r_{0}) \leq \max\left\{ \max_{r \in (r_{0},1)} 16 {\cal{B}}^{\prime}_{\hat{n}}\left(\left\lceil \frac{1}{r} \right\rceil,\frac{1}{20}\right), 512 \right\}.
$$
Finally, note that for any $m \in \nats$ and sets $C, S \in (\cX \times \cY)^{m}$, 
$\VS_{\F \cup \{\target\},C,\target} = \VS_{\F,C,\target} \cup \{\target\}$ and $\VS_{\F \cup \{\target\},S,\target} = \VS_{\F,S,\target} \cup \{\target\}$,
so that $\VS_{\F \cup \{\target\},C,\target} = \VS_{\F \cup \{\target\},S,\target}$ if and only if $\VS_{\F,C,\target} = \VS_{\F,S,\target}$.
Thus, $\hat{n}(\F\cup\{\target\},S_{m},\target) = \hat{n}(\F,S_{m},\target)$,
so that ${\cal{B}}^{\prime}_{\hat{n}}\left(\left\lceil \frac{1}{r} \right\rceil, \frac{1}{20}\right) = \Bound{\hat{n}}{\left\lceil \frac{1}{r} \right\rceil}{\frac{1}{20}}$,
which implies the result.
\end{proof}

\subsection{Label complexity bound for agnostic active learning}

$A^2$ (\emph{Agnostic Active}) was the first general-purpose agnostic
active learning algorithm with proven improvement in error guarantees
compared to passive learning.  The original work of \citet*{balcan2006agnostic},
which first introduced this algorithm, also provided specialized proofs that
the algorithm achieves an exponential label complexity speedup (for the low accuracy regime)
compared to passive learning for a few simple cases, including: threshold
functions, and homogenous linear separators under a uniform distribution
over the sphere.  Additionally, \citet*{Hanneke07}
provided a general bound on the label complexity of $A^2$, expressed in
terms of the disagreement coefficient, so that any bound on the disagreement
coefficient translates into a bound on the label complexity of agnostic
active learning with $A^2$.  Inspired by the $A^2$ algorithm, other noise-robust
active learning algorithms have since been proposed, with improved
label complexity bounds compared to those proven by \citet*{Hanneke07}
for $A^2$, while still expressed in terms of the disagreement coefficient
\citep*[see e.g.,][]{dasgupta2007general,hannekestatistical}.  As an example
of such results, the following result was proven by \citet*{dasgupta2007general}.

\begin{theorem}[\citealp*{dasgupta2007general}]
\label{thm:A2}
There exists a finite universal constant $c > 0$ such that,
for any $\epsilon,\delta \in (0,1/2)$,
using hypothesis class $\F$,  and given the input $\delta$ and a budget $n$ on the number of label requests,
the active learning algorithm of \citet*{dasgupta2007general} requests at most $n$ labels,\footnote{This result
applies to a slightly modified variant of the algorithm of \citet*{dasgupta2007general}, studied by \citet*{hannekeAnnals11},
which terminates after a given number of label requests, rather than after a given number of unlabeled samples.
The same is true of Theorem~\ref{thm:A2-td-bound} and Corollary~\ref{cor:linsep-gaussian-agnostic}.}
and if 
$$
n \geq c \dc(\er(\target) + \epsilon)\left(\frac{\er(\target)^2}{\epsilon^2} + 1\right)\left(d\log\left(\frac{1}{\epsilon}\right) + \log\left(\frac{1}{\delta}\right)\right)\log\left(\frac{1}{\epsilon}\right),
$$
then with probability at least $1-\delta$, the classifier $\hat{f} \in \F$ it produces satisfies
$$
\er(\hat{f}) \leq \er(\target) + \epsilon.
$$
\end{theorem}

Combined with the results above, this implies the following theorem.

\begin{theorem}
\label{thm:A2-td-bound}
There exists a finite universal constant $c > 0$ such that,
for any $\epsilon,\delta \in (0,1/2)$,
using hypothesis class $\F$,  and given the input $\delta$ and a budget $n$ on the number of label requests,
the active learning algorithm of \citet*{dasgupta2007general} requests at most $n$ labels,
and if 
$$
n \geq c\left(\max_{r > \er(\target)+\epsilon} \Bound{\hat{n}}{\left\lceil\frac{1}{r}\right\rceil}{\frac{1}{20}} + 1\right)\left(\frac{\er(\target)^2}{\epsilon^2} + 1\right)\left(d\log\left(\frac{1}{\epsilon}\right) + \log\left(\frac{1}{\delta}\right)\right)\log\left(\frac{1}{\epsilon}\right),
$$
then with probability at least $1-\delta$, the classifier $\hat{f} \in \F$ it produces satisfies
$$
\er(\hat{f}) \leq \er(\target) + \epsilon.
$$
\end{theorem}
\begin{proof}
By Lemma~\ref{lem:td-dc-bound-agnostic},
\begin{align*}
\dc(\er(\target)+\epsilon)
& \leq \max\left\{ \max_{r \in (\er(\target)+\epsilon,1)} 16 \Bound{\hat{n}}{\left\lceil \frac{1}{r} \right\rceil}{\frac{1}{20}}, 512\right\}
\\ & \leq 512 \left( \max_{r > \er(\target)+\epsilon} \Bound{\hat{n}}{\left\lceil \frac{1}{r} \right\rceil}{\frac{1}{20}} + 1\right).
\end{align*}
Plugging this into Theorem~\ref{thm:A2} yields the result.
\end{proof}


Interestingly, from the perspective of bounding the label complexity of agnostic active learning in general, 
the result in Theorem~\ref{thm:A2-td-bound} sometimes improves over a related bound proven by \citet*{Hanneke_COLT_07}
(for a different algorithm).  Specifically, compared to the result of \citet*{Hanneke_COLT_07}, this result 
maintains an interesting dependence on $\target$, whereas the bound of \citet*{Hanneke_COLT_07}
effectively replaces the factor $\Bound{\hat{n}}{\lceil 1/r \rceil}{1/20}$ with the maximum of 
this quantity over the choice of $\target$.\footnote{There are a few other differences, which are usually minor.
For instance, the bound of \citet*{Hanneke_COLT_07} uses $r \approx \er(\target)+\epsilon$ rather than maximizing
over $r > \er(\target)+\epsilon$.  That result additionally replaces ``$1/20$'' with a value $\delta^{\prime} \approx \delta / n$.}
Also, while the result of \citet*{Hanneke_COLT_07} is proven for an algorithm that requires explicit access to a 
value $\eta \approx \er(\target)$ to obtain the stated label complexity, the label complexity in Theorem~\ref{thm:A2-td-bound}
is achieved by the algorithm of \citet*{dasgupta2007general}, which requires no such extra parameters.

As an application of Theorem~\ref{thm:A2-td-bound}, we have the following corollary.

\begin{corollary}
\label{cor:linsep-gaussian-agnostic}
For $t,k \in \nats$ and $c \in (0,\infty)$, there exists a finite constant $c_{k,t,c} > 0$ 
such that, for $\F$ the class of linear separators on $\reals^{k}$, 
and for $P$ with marginal distribution over $\cX$ that is a mixture of $t$ 
multivariate normal distributions with diagonal covariance matrices of full rank,
for any $\epsilon,\delta \in (0,1/2)$ with 
$\epsilon \geq \frac{\er(\target)}{c}$, 
using hypothesis class $\F$, and given the input $\delta$ and a budget $n$ on the number of label requests,
the active learning algorithm of \citet*{dasgupta2007general} requests at most $n$ labels, and if
\begin{equation*}
n \geq c_{k,t,c} \left(\log\left( \frac{1}{\epsilon} \right)\right)^{k+1} \log\left(\frac{1}{\delta}\right),
\end{equation*}
then with probability at least $1-\delta$, the classifier $\hat{f} \in \F$ it produces satisfies
$\er(\hat{f}) \leq \er(\target) + \epsilon$.
\end{corollary}
\begin{proof}
Let $\F$ and $P$ be as described above.
First, we argue that $\target \in \F$.
Fix any classifier $f$ with $\inf_{h \in \F} P((x,y) : h(x) \neq f(x))=0$.
%
There must exist a sequence $\{ (b^{(t)},w^{(t)}_{1},\ldots,w^{(t)}_{k}) \}_{k=1}^{\infty}$ in $\reals^{k+1}$ with $\sum_{i=1}^{k} (w^{(t)}_{i})^{2} = 1$ for all $t$,
s.t. $P\left((x_1,\ldots,x_k,y) : \sign\left( b^{(t)} + \sum_{i=1}^{k} x_{i} w^{(t)}_{i} \right) \neq f(x_{1},\ldots,x_{k}) \right) \to 0$.
If $\limsup\limits_{t \to \infty} b^{(t)} = \infty$, then $\exists t_{j} \to \infty$ with $b^{(t_{j})} \to \infty$,
and since every $(x_1,\ldots,x_k) \in \reals^{k}$ has $\sum_{i=1}^{k} x_{i} w^{(t)}_{i} \geq - \|x\|$, 
we have that $b^{(t_{j})} + \sum_{i=1}^{k} x_{i} w^{(t_{j})}_{i} \to \infty$,
which implies $\sign\left( b^{(t_{j})} + \sum_{i=1}^{k} x_{i} w^{(t_{j})}_{i} \right) \to 1$ for all $(x_1,\ldots,x_k) \in \reals^{k}$.
Similarly, if $\liminf\limits_{t \to \infty} b^{(t)} = -\infty$, then $\exists t_{j} \to \infty$ with
$\sign\left( b^{(t_{j})} + \sum_{i=1}^{k} x_{i} w^{(t_{j})}_{i} \right) \to -1$ for all $(x_1,\ldots,x_k) \in \reals^{k}$.
Otherwise, if $\limsup_{t \to \infty} b^{(t)} < \infty$ and $\liminf_{t \to \infty} b^{(t)} > -\infty$, then the sequence 
$\{ (b^{(t)},w^{(t)}_{1},\ldots,w^{(t)}_{k}) \}_{t=1}^{\infty}$ is \emph{bounded} in $\reals^{k+1}$.  Therefore, the 
Bolzano-Weierstrass Theorem implies it contains a convergent subsequence:
that is, $\exists t_{j} \to \infty$ s.t. $(b^{(t_{j})},w^{(t_{j})}_{1},\ldots,w^{(t_{j})}_{k})$ converges.
Furthermore, since $\{ w \in \reals^{k} : \|w\|=1 \}$ is closed, and $\{b^{(t)} : t \in \nats\} \subseteq [ \inf_t b^{(t)}, \sup_t b^{(t)} ]$,
which is a closed subset of $\reals$, 
$\exists (b,w_{1},\ldots,w_{k}) \in \reals^{k+1}$ with $\sum_{i=1}^{k} w_{i}^{2} = 1$ such that 
$(b^{(t_{j})},w^{(t_{j})}_{1},\ldots,w^{(t_{j})}_{k}) \to (b,w_{1},\ldots,w_{k})$.
Continuity of linear functions implies, $\forall (x_1,\ldots,x_k) \in \reals^{k}$, 
$b^{(t_{j})} + \sum_{i=1}^{k} x_{i} w^{(t_{j})}_{i} \to b + \sum_{i=1}^{k} x_{i} w_{i}$.
Therefore, every $(x_1,\ldots,x_k) \in \reals^{k}$ with $b + \sum_{i=1}^{k} x_{i} w_{i} > 0$ has
$\sign\left( b^{(t_{j})} + \sum_{i=1}^{k} x_{i} w^{(t_{j})}_{i} \right) \to 1$, and every $(x_1,\ldots,x_k) \in \reals^{k}$ with 
$b + \sum_{i=1}^{k} x_{i} w_{i} < 0$ has $\sign\left( b^{(t_{j})} + \sum_{i=1}^{k} x_{i} w^{(t_{j})}_{i} \right) \to -1$.
Since $P\left( (x_1,\ldots,x_k,y) : b + \sum_{i=1}^{k} x_{i} w_{i} = 0 \right) = 0$, this implies 
$(x_1,\ldots,x_k) \mapsto \sign\left( b^{(t_{j})} + \sum_{i=1}^{k} x_{i} w^{(t_{j})}_{i} \right)$ converges to 
$(x_1,\ldots,x_k) \mapsto \sign\left( b + \sum_{i=1}^{k} x_{i} w_{i} \right)$ almost surely [$P$].

Thus, in each case, 
$\exists t_{j} \to \infty$ and $h \in \F$ s.t.
$(x_1,\ldots,x_k) \mapsto \sign\left( b^{(t_{j})} + \sum_{i=1}^{k} x_{i} w^{(t_{j})}_{i} \right)$ 
converges to $h$ a.s. [$P$]. 
Since convergence almost surely implies convergence in probability, 
we have
$P\left( (x_1,\ldots,x_k,y) : \sign\left( b^{(t_{j})} + \sum_{i=1}^{k} x_{i} w^{(t_{j})}_{i} \right) \neq h(x_1,\ldots,x_k) \right) \to 0$.
Furthermore, by assumption, $P\left( (x_1,\ldots,x_k,y) : \sign\left( b^{(t_{j})} + \sum_{i=1}^{k} x_{i} w^{(t_{j})}_{i} \right) \neq f(x_1,\ldots,x_k) \right) \to 0$ as well.
Thus, a union bound implies $P( (x,y) : h(x) \neq f(x) ) = 0$.
In particular, we have that for any $f$ with $\inf_{g \in \F} P((x,y) : g(x) \neq f(x)) = 0$ and $\er(f) = \inf_{g \in \F} \er(g)$, 
$\exists h \in \F$ with $P((x,y) : f(x) \neq h(x)) = 0$, and hence $\er(h) = \inf_{g \in \F} \er(g)$.  Thus, we may assume 
$\target \in \F$ in this setting.


Therefore, in this scenario, Theorem~\ref{thm:linsep-gaussian} implies
\begin{equation*}
\max_{r > \er(\target)+\epsilon} \Bound{\hat{n}}{\left\lceil \frac{1}{r} \right\rceil}{\frac{1}{20}} + 1
\leq c_{k,t}^{(1)} \left(\log\left( \frac{2}{\er(\target)+\epsilon} \right)\right)^{k-1},
\end{equation*}
for an appropriate $(k,t)$-dependent constant $c_{k,t}^{(1)} \in (0,\infty)$.
Plugging this into Theorem~\ref{thm:A2-td-bound}, and recalling that the VC dimension of the class of linear classifiers in $\reals^{k}$ is $k+1$ \citep*[see e.g.,][]{AnthoB99},
we get a bound on the number of label requests of
\begin{align*}
& c_{k,t}^{(2)}\left(\log\left( \frac{2}{\er(\target)+\epsilon} \right)\right)^{k-1} \left(\frac{\er(\target)^2}{\epsilon^2} + 1\right)\left(k\log\left(\frac{1}{\epsilon}\right) + \log\left(\frac{1}{\delta}\right)\right)\log\left(\frac{1}{\epsilon}\right)
\\ & \leq c_{k,t}^{(3)} \left(\log\left(\frac{1}{\epsilon}\right)\right)^{k+1} \left(\frac{\er(\target)^2}{\epsilon^2} + 1\right)\left(k + \log\left(\frac{1}{\delta}\right)\right),
\end{align*}
for appropriate $(k,t)$-dependent constants $c_{k,t}^{(2)},c_{k,t}^{(3)} \in (0,\infty)$.
Since (by assumption) $\epsilon \geq \frac{\er(\target)}{c}$, this is at most
\begin{equation*}
c_{k,t,c}^{(4)} \left(\log\left(\frac{1}{\epsilon}\right)\right)^{k+1} \left(k + \log\left(\frac{1}{\delta}\right)\right)
\leq c_{k,t,c}^{(5)} \left(\log\left(\frac{1}{\epsilon}\right)\right)^{k+1} \log\left(\frac{1}{\delta}\right),
\end{equation*}
for appropriate $(k,t,c)$-dependent constants $c_{k,t,c}^{(4)},c_{k,t,c}^{(5)} \in (0,\infty)$.
Thus, taking $c_{k,t,c} = c_{k,t,c}^{(5)}$ establishes the result.
\end{proof}

An analogous result can be shown for the problem of learning axis-aligned rectangles via Theorem~\ref{thm:rectangles}.

\subsection{Label complexity bound under Mammen-Tsybakov noise}

Since the original work on agnostic active learning discussed above, there have been several other
analyses, expressing the noise conditions in terms of quantities other than the noise rate $\er(\target)$.
Specifically, the following condition of \citet*{MammenT99} has been studied for several algorithms
\citep*[see e.g.,][]{balcan:07,hannekeAnnals11,koltchinskii:10,Hanneke11,hanneke:12b,hannekestatistical,Beygelzimer10,hsu:thesis}.

\begin{condition}[\citealp*{MammenT99}]
\label{cond:noise}
For some $a \in [1,\infty)$ and $\alpha \in [0,1]$, for every $f \in \F$,
$$
\Pr(f(X) \neq \target(X)) \leq a(\er(f) - \er(\target))^ \alpha.
$$
\end{condition}

In particular, for a variant of $A^2$ known as RobustCAL$_{\delta}$, studied by \citet*{Hanneke11,hannekestatistical} and \citet*{hanneke:12b},
the following result is known \citep*[due to][]{hanneke:12b}. 

\begin{theorem}[\citealp*{hanneke:12b}]
\label{thm:robust}
There exists a finite universal constant $c > 0$ such that,
for any $\epsilon,\delta \in (0,1/2)$, for any $n,u \in \nats$, given the arguments $n$ and $u$, 
the RobustCAL$_{\delta}$ algorithm requests at most $n$ labels, and if $u$ is sufficiently large,
and
\begin{equation*}
n \geq c a^2\dc(a\epsilon^\alpha)\left(\frac{1}{\epsilon}\right)^{2-2\alpha}\left(d\log\left(e\dc\left( a\epsilon^\alpha \right)\right) + \log \left(\frac{\log(1/\epsilon)}{\delta}\right)\right)\log\left(\frac{1}{\epsilon}\right),
\end{equation*}
for $a$ and $\alpha$ as in Condition~\ref{cond:noise},
then with probability at least $1-\delta$, the classifier $\hat{f} \in \F$ it returns satisfies
$\er(\hat{f}) \leq \er(\target) + \epsilon$.
\end{theorem}

Combined with Theorem~\ref{thm:td-dc-bound}, this implies the following theorem.

\begin{theorem}
\label{thm:robust-td-bound}
There exists a finite universal constant $c > 0$ such that,
for any $\epsilon,\delta \in (0,1/2)$, for any $n,u \in \nats$, given the arguments $n$ and $u$, 
the RobustCAL$_{\delta}$ algorithm requests at most $n$ labels, and if $u$ is sufficiently large,
and
\begin{equation*}
n \geq c a^2\left( \max_{r > a \epsilon^{\alpha}} \Bound{\hat{n}}{\left\lceil\frac{1}{r}\right\rceil}{\frac{1}{20}}+1\right)\left(\frac{1}{\epsilon}\right)^{2-2\alpha}\left(d\log\left(\frac{1}{\epsilon}\right) + \log \left(\frac{1}{\delta}\right)\right)\log\left(\frac{1}{\epsilon}\right),
\end{equation*}
for $a$ and $\alpha$ as in Condition~\ref{cond:noise},
then with probability at least $1-\delta$, the classifier $\hat{f} \in \F$ it returns satisfies
$\er(\hat{f}) \leq \er(\target) + \epsilon$.
\end{theorem}

In particular, reasoning as in Corollary~\ref{cor:linsep-gaussian-agnostic} above, 
Theorem~\ref{thm:robust-td-bound} implies the following corollary.

\begin{corollary}
\label{cor:robust_linear}
For $t,k \in \nats$ and $a \in [1,\infty)$, there exists a finite constant $c_{k,t,a} > 0$ 
such that, for $\F$ the class of linear separators on $\reals^{k}$, 
and for $P$ satisfying Condition~\ref{cond:noise} with $\alpha = 1$ and the given value of $a$,
and with marginal distribution over $\cX$ that is a mixture of $t$ 
multivariate normal distributions with diagonal covariance matrices of full rank,
for any $\epsilon,\delta \in (0,1/2)$, for any $n,u \in \nats$, given the arguments $n$ and $u$,
the RobustCAL$_{\delta}$ algorithm requests at most $n$ labels, and if $u$ is sufficiently large, 
and 
\begin{equation*}
n \geq c_{k,t,a} \left(\log\left(\frac{1}{\epsilon}\right)\right)^{k+1}\log\left(\frac{1}{\delta}\right),
\end{equation*}
then with probability at least $1-\delta$, the classifier $\hat{f} \in \F$ it returns satisfies
$\er(\hat{f}) \leq \er(\target) + \epsilon$.
\end{corollary}

Corollary~\ref{cor:robust_linear} proves an exponential label complexity speedup in the asymptotic dependence on $\epsilon$ compared to passive learning, 
for which there is a lower bound on the label complexity of $\Omega(1/\epsilon)$ in the worst case over these distributions \citep*{long:95}.
\begin{remark}
Condition~\ref{cond:noise} can be satisfied with $\alpha = 1$ if the Bayes optimal classifier is in $\F$ and the source distribution satisfies \emph{Massart noise} \citep*{massart2006risk}:
\begin{equation*}
\Pr\left(|P(Y=1|X=x) - 1/2| < 1/(2a)\right) = 0.
\end{equation*}
For example, if the data was generated by some unknown linear hypothesis with label noise (probability to flip any label) of up to $(a-1)/2a$, then $P$ satisfies the requirements of Corollary~\ref{cor:robust_linear}.
\end{remark}

\section*{Acknowledgements}
R. El-Yaniv's research is funded by the 
Intel Collaborative Research Institute for Computational Intelligence (ICRI-CI).

\appendix

\section{Analysis of CAL via the Disagreement Coefficient}
\label{sec:dc-appendix}

The following result was first established by \citep*[][page 1213]{gine:06}, with slightly different constant factors.
The version stated here is directly from \citet*[][Section 2.9]{hannekethesis09}, who also presents a simple and direct proof.

\begin{lemma}[\citealp*{gine:06,hannekethesis09}]
\label{lem:gk-passive}
For any $t \in \nats$ and $\delta \in (0,1)$, with probability at least $1-\delta$, 
\begin{equation*}
\sup_{h \in \VS_{\F,S_{t}}} \er(h) \leq 
\frac{24}{t} \left( d \ln\left( 880 \cdot \dc( d / t ) \right) + \ln\left(\frac{12}{\delta}\right) \right).
\end{equation*}
\end{lemma}

The following result is implicit in a proof of \citet*{hannekeAnnals11}; for completeness, we present a formal proof here.

\begin{lemma}[\citealp*{hannekeAnnals11}]
\label{lem:dc-N-bound} 
There exists a finite universal constant $c_{0} > 0$ such that,
$\forall \delta \in (0,1)$, $\forall m \in \nats$ with $m \geq 2$,
\begin{equation*}
\Bound{N}{m}{\delta} \leq c_{0} \dc(d/m) \left( d \ln\left( e\dc(d/m) \right) + \ln\left( \frac{\log_{2}(m)}{\delta} \right) \right) \log_{2}(m).
\end{equation*}
\end{lemma}
\begin{proof}
The result trivially holds for $m=2$, taking any $c_{0} \geq 2$.  Otherwise, suppose $m \geq 3$.
Note that, for any $t \in \nats$, 
\begin{equation}
\label{eqn:dc-N-bound-1}
\frac{24}{t} \left( d \ln\left( 880 \dc( d / t ) \right) + \ln\left( \frac{24 \log_{2}(m)}{\delta} \right) \right)
\leq \frac{c_{1}}{t} \left( d \ln\left( e \dc( d / t ) \right) + \ln\left( \frac{2 \log_{2}(m)}{\delta} \right) \right),
\end{equation}
for some universal constant $c_{1} \in [1,\infty)$ (e.g., taking $c_{1} = 168$ suffices).
Thus, letting $r_{t}$ denote the expression on the right hand side of \eqref{eqn:dc-N-bound-1},
Lemma~\ref{lem:gk-passive} implies that, 
for any $t \in \nats$, 
with probability at least $1-\delta / (2\log_{2}(m))$,
\begin{equation*}
\sup_{h \in \VS_{\F,S_{t}}} \er(h) \leq r_{t}.
\end{equation*}
By a union bound, this holds for all $t \in \{2^{i} : i \in \{1,\ldots,\lceil \log_{2}(m) \rceil - 1\} \}$ with probability at least $1-\delta/2$.
In particular, on this event, we have
\begin{equation*}
N(m;S_{m}) \leq 2 + \sum_{i=1}^{\lceil \log_{2}(m) \rceil - 1} \sum_{t = 2^{i}+1}^{2^{i+1}} \ind_{\DIS(\Ball(\target,r_{2^{i}}))}(x_{t}).
\end{equation*}
A Chernoff bound implies that, with probability at least $1-\delta/2$, the right hand side is at most
\begin{align*}
& \log_{2}\left( \frac{8}{\delta} \right) + 2 e \sum_{i=1}^{\lceil \log_{2}(m) \rceil - 1} 2^{i} \PDIS \Ball(\target,r_{2^{i}})
\\ & \leq \log_{2}\left( \frac{8}{\delta} \right) + 2 e \sum_{i=1}^{\lceil \log_{2}(m) \rceil - 1} 2^{i} \dc\left(r_{2^{i}}\right) r_{2^{i}}
\\ & \leq \log_{2}\left( \frac{8}{\delta} \right) + 2 e c_{1} \sum_{i=1}^{\lceil \log_{2}(m) \rceil - 1} \dc\left(d 2^{-i}\right) \left( d \ln\left( e \dc\left( d 2^{-i} \right) \right) + \ln\left( \frac{2 \log_{2}(m)}{\delta} \right) \right)
\\ & \leq 4 e c_{1} \dc(d/m) \left( d \ln\left( e\dc( d / m ) \right) + \ln\left( \frac{\log_{2}(m)}{\delta} \right) \right) \log_{2}(m).
\end{align*}
Letting $c_{0} = 4ec_{1}$, the result holds by a union bound and minimality of $\Bound{N}{m}{\delta}$.
\end{proof}

The following result is taken from the work of \citet*[][Proof of Theorem 1]{hannekeAnnals11}; see also \citet*{hannekestatistical}
for a theorem and proof expressed in this exact form. 

\begin{lemma}[\citealp*{hannekeAnnals11}]
\label{lem:dc-LC-bound}
There exists a finite universal constant $c_{0} > 0$ such that,
$\forall \epsilon,\delta \in (0,1/2]$,
\begin{equation*}
\LC(\epsilon,\delta) \leq c_{0} \dc(\epsilon) \left( d \ln( e \dc(\epsilon) ) + \ln\left( \frac{\log_{2}(1/\epsilon)}{\delta} \right) \right) \log_{2}\left(\frac{1}{\epsilon}\right).
\end{equation*}
\end{lemma}

The next result is taken from the work of \citet*[][Corollary 39]{el2012active}.

\begin{lemma}[\citealp*{el2012active}]
\label{cor:boundingDiss}
For any $r_{0} \in (0,1)$,
\begin{equation*}
\dc(r_0) \leq \max\left\{\sup_{r \in (r_0, 1/2)} \frac{7 \cdot \Bound{\PDIS}{\lfloor 1/r \rfloor}{1/9}}{r},2\right\}.
\end{equation*}
\end{lemma}

\section{Separation from the Previous Analyses}
\label{app:looseness-examples}

There are simple examples showing that sometimes $\Bound{\hat{n}}{m}{\delta} \approx \dc(1/m)$,
so that the upper bound $\LC(\epsilon,\delta) \leq c_{0} d \dc(\epsilon) \polylog\left(\frac{1}{\epsilon\delta}\right)$
in Lemma~\ref{lem:dc-LC-bound} is off by a factor of $d$ compared to Theorem~\ref{thm:cal-label-complexity} in those cases (aside from logarithmic factors).
For instance, consider the class of unions of $k$ intervals, where $k \in \nats$, $\cX = [0,1]$, and $\F = \{ x \mapsto 2\ind_{\bigcup_{i=1}^{k} [z_{2i-1},z_{2i}]}(x)-1  : 0 < z_{1} < \cdots < z_{2k} < 1 \}$.
Suppose the data distribution $P$ has a uniform marginal distribution over $\cX$,
and has $\target = 2\ind_{\bigcup_{i=1}^{k} [z_{2i-1}^{*},z_{2i}^{*}]}-1$, where $z_{i}^{*} = \frac{i}{2k+1}$ for $i \in \{1,\ldots,2k\}$.
In this case, for $r_{0} \geq 0$, $\dc(r_{0})$ is within a factor of $2$ of $\min\left\{ \frac{1}{r_{0}}, 4k\right\}$ \citep*[see e.g.,][]{BHV:10,Hanneke11}.
However, for any $m \in \nats$ with $m \geq (2k+1) \ln\left( \frac{2k+1}{\delta} \right)$, with probability at least $1-\delta$
we have for each $i \in \{0,\ldots,2k\}$, at least one $j \leq m$ has $\frac{i}{2k+1} < x_{j} < \frac{i+1}{2k+1}$, and no $j \leq m$ has $x_{j} = \frac{i}{2k+1}$;
in this case, $\VScomp_{S_{m}}$ is constructed as follows; for each $i \in \{1,\ldots,2k\}$, we include in $\VScomp_{S_{m}}$
the point $(x_j,y_j)$ with largest $x_j$ less than $\frac{i}{2k+1}$ and the point $(x_j,y_j)$ with smallest $x_j$ greater than $\frac{i}{2k+1}$.
The number of points in this set $\VScomp_{S_{m}}$ is at most $4k$.
Therefore, for any $m \in \nats$, we have
$\Bound{\hat{n}}{m}{\delta} \leq \min\left\{ m, \max\left\{\left\lceil (2k+1)\ln\left( \frac{2k+1}{\delta} \right) \right\rceil, 4k\right\} \right\}$.
In particular, noting that $d = 2k$ here, we have that for $\epsilon < 1/k$,
the bound on $\LC(\epsilon,\delta)$ in Lemma~\ref{lem:dc-LC-bound} has a $\tilde{\Theta}(k^{2})$ dependence on $k$,
while the upper bound on $\LC(\epsilon,\delta)$ in Theorem~\ref{thm:cal-label-complexity} has only a
$\tilde{\Theta}(k)$ dependence on $k$, which matches the lower bound in Theorem~\ref{thm:cal-label-complexity} (up to logarithmic factors).

Aside from the disagreement coefficient, the other technique in the existing literature for bounding the label complexity of \CAL~is
due to \citet*{ElYaniv10,el2012active}, based on a quantity they call the \emph{characterizing set complexity}, denoted $\gamma(\F,\hatn{m})$.
Formally, for $n \in \nats$, let $\gamma(\F,n)$ denote the VC dimension of the collection of sets
$\{ \DIS(\VS_{\F,S}) : S \in (\cX \times \cY)^{n} \}$.	Then \citet*{el2012active} prove the following bound, for a universal constant $c \in (0,\infty)$.\footnote{This result
can be derived from their Theorem 15 via reasoning analogous to the derivation of
Theorem~\ref{thm:cal-label-complexity} from Lemma~\ref{lem:data-dependent-coverage} above.}
\begin{multline}
\label{eqn:ew12}
\LC(\epsilon,\delta)
\leq c \Bigg( \max_{m \leq M(\epsilon,\delta/2)} \gamma\left(\F, \Bound{\hat{n}}{m}{\delta}\right) \ln\left( \frac{ e m }{\gamma\left(\F, \Bound{\hat{n}}{m}{\delta}\right)} \right)
\\ + \ln\left( \frac{\log_{2}(2 M(\epsilon,\delta/2))}{\delta} \right)\Bigg)\log_{2}(2M(\epsilon,\delta/2)).
\end{multline}
We can immediately note that $\gamma\left(\F, \Bound{\hat{n}}{m}{\delta}\right) \geq \Bound{\hat{n}}{m}{\delta}-1$;
specifically, for any $S \in (\cX \times \cY)^{m}$, letting $\{(x_{i_{1}},y_{i_{1}}),\ldots,(x_{i_{\hatn{m}}},y_{i_{\hatn{m}}})\} = \VScomp_{S}$,
we have that $\{x_{i_{2}},\ldots,x_{i_{\hatn{m}}}\}$ is shattered by $\{ \DIS(\VS_{\F,S^{\prime}}) : S^{\prime} \in (\cX \times \cY)^{\hatn{m}} \}$,
since letting $S^{\prime}$ be any subset of $\{(x_{i_{2}},y_{i_{2}}),\ldots,(x_{i_{\hatn{m}}},y_{i_{\hatn{m}}})\}$ (filling in the remaining elements as copies of $(x_{i_{1}},y_{i_{1}})$ to make $S^{\prime}$ of size $\hatn{m}$),
\begin{equation*}
\{(x_{i_{2}},y_{i_{2}}),\ldots,(x_{i_{\hatn{m}}},y_{i_{\hatn{m}}})\} \cap (\DIS(\VS_{\F,S^{\prime}})\times\cY) = \{(x_{i_{2}},y_{i_{2}}),\ldots,(x_{i_{\hatn{m}}},y_{i_{\hatn{m}}})\} \setminus S^{\prime},
\end{equation*}
since otherwise, the $(x_{i_j},y_{i_j})$ in $\{(x_{i_2},y_{i_2}),\ldots,(x_{i_{\hatn{m}}},y_{i_{\hatn{m}}})\} \setminus S^{\prime}$ not in $\DIS(\VS_{\F,S^{\prime}})\times\cY$ 
would have $x_{i_j} \notin \DIS( \VS_{\F,\VScomp_{S} \setminus \{(x_{i_j},y_{i_j})\}} )$, so that $\VS_{\F,\VScomp_{S} \setminus \{(x_{i_j},y_{i_j})\}} = \VS_{\F,\VScomp_{S}} = \VS_{\F,S}$,
contradicting minimality of $\VScomp_{S}$.
Therefore, $\gamma\left(\F,\hatn{m}\right) \geq \hatn{m}-1$.
Then noting that $\gamma\left(\F,n\right)$ is monotonic in $n$,
we find that $\gamma\left(\F,\Bound{\hat{n}}{m}{\delta}\right)$ is a minimal $1-\delta$ confidence bound on $\gamma\left(\F,\hatn{m}\right)$,
which implies $\gamma\left(\F,\Bound{\hat{n}}{m}{\delta}\right) \geq \Bound{\hat{n}}{m}{\delta}-1$.

One can also give examples where the gap between $\Bound{\hat{n}}{m}{\delta}$ and $\gamma(\F,\Bound{\hat{n}}{m}{\delta}$ is large, 
for instance where $\gamma(\F,\Bound{\hat{n}}{m}{\delta}) \geq d$ while $\Bound{\hat{n}}{m}{\delta} = 2$ for large $m$.
%
For instance, consider $\cX$ that has $d$ points $w_1,\ldots,w_d$ and $2^{d+1}$ additional points $x_{I}$ and $z_{I}$ indexed by the sets $I \subseteq \{1,\ldots,d\}$, and
say $\F$ is the space of classifiers $\{h_{J} : J \subseteq \{1,\ldots,d\}\}$, where for each $J \subseteq \{1,\ldots,d\}$,
$\{x : h_{J}(x) = +1\} = \{w_{i} : i \in J\} \cup \{ x_{I} : I \subseteq J \} \cup \{ z_{I} : I \subseteq \{1,\ldots,d\} \setminus J \}$;
in particular, the classification on $w_1,\ldots,w_d$ determines the classification on the remaining $2^{d+1}$ points,
and $\{w_1,\ldots,w_d\}$ is shatterable, so that $|\F|=2^{d}$, and the VC dimension of $\F$ is $d$.
Let $P$ be a distribution that has a uniform marginal distribution over the $2^{d+1}+d$ points in $\cX$,
and satisfies the realizable case assumption (i.e., $\P(Y=\target(X)|X)=1$, for some $\target \in \F$).
For any integer $m \geq (2^{d+1}+d) \ln( 2/\delta )$, with probability at least $1-\delta$,
we have $(x_{\{i \leq d : \target(w_{i})=+1\}},+1) \in S_{m}$ and $(z_{\{i \leq d : \target(w_{i})=-1\}},+1) \in S_{m}$.
Since every $h_{J} \in \F$ with $h_{J}(x_{\{i \leq d : \target(w_{i})=+1\}})=+1$ has $\{i \leq d : \target(w_{i})=+1\} \subseteq J = \{i \leq d : h_{J}(w_{i})=+1\}$,
and every $h_{J} \in \F$ with $h_{J}(z_{\{i \leq d : \target(w_{i})=-1\}})=+1$ has $\{i \leq d : \target(w_{i})=-1\} \subseteq \{1,\ldots,d\} \setminus J = \{i \leq d : h_{J}(w_{i})=-1\}$,
so that $\{i \leq d : \target(w_{i})=+1\} \supseteq \{i \leq d : h_{J}(w_{i})=+1\}$,
we have that every $h_{J} \in \F$ with both $h_{J}(x_{\{i \leq d : \target(w_{i})=+1\}})=+1$ and $h_{J}(z_{\{i \leq d : \target(w_{i})=-1\}})=+1$ has
$\{i \leq d : h_{J}(w_{i}) = +1\} = \{i \leq d : \target(w_{i}) = +1\}$. Since classifiers in $\F$ are completely determined by their
classification of $\{w_{1},\ldots,w_{d}\}$, this implies $h_{J} = \target$.  Therefore, letting $\VScomp_{S_{m}} = \{(x_{\{i \leq d : \target(w_{i})=+1\}},+1),(z_{\{i \leq d : \target(w_{i})=-1\}},+1)\}$,
we have $\VS_{\F,\VScomp_{S_{m}}} = \VS_{\F,S_{m}}$, so that $\hatn{m} \leq 2$ (in fact, one can easily show $\hatn{m}=2$ in this case).
Thus, for large $m$, $\Bound{\hat{n}}{m}{\delta} \leq 2$.
However, for any $I \subseteq \{1,\ldots,d\}$, letting $S = \{ (x_{\{1,\ldots,d\}\setminus I},+1) \}$,
we have $h_{\{1,\ldots,d\} \setminus I} \in \VS_{\F,S}$,
every $h \in \VS_{\F,S}$ has $h(w_{i}) = +1$ for every $i \in \{1,\ldots,d\} \setminus I$,
and every $i \in I$ has $h_{(\{1,\ldots,d\}\setminus I) \cup \{i\}} \in \VS_{\F,S}$,
so that $\DIS(\VS_{\F,S}) \cap \{w_{1},\ldots,w_{d}\} = \{w_{i} : i \in I\}$;
therefore, the VC dimension of $\{ \DIS(\VS_{\F,\{x\}}) : x \in \cX \}$ is at least $d$:
that is, $\gamma(\F,1) \geq d$.
Since we have $\hatn{m} \geq 1$ whenever $S_{m}$ contains any point other than $x_{\{\}}$ and $z_{\{\}}$,
and this happens with probability at least $1 - (2 / (2^{d+1}+d))^{m} \geq 1-\delta > \delta$ (when $\delta < 1/2$),
this implies we have $\gamma(\F,\hatn{m}) \geq \gamma(\F,1) \geq d$ with probability greater than $\delta$,
which (by monotonicity of $\gamma(\F,\cdot)$) implies $\gamma(\F,\Bound{\hat{n}}{m}{\delta}) \geq d$.

This is not quite strong enough to show a gap between \eqref{eqn:ew12} and Theorem~\ref{thm:cal-label-complexity},
since the bounds in Theorem~\ref{thm:cal-label-complexity} require us to \emph{maximize} over the value of $m$,
which would therefore also include values $\Bound{\hat{n}}{m}{\delta}$ for $m < (2^{d+1}+d) \ln( 2/\delta )$.
To exhibit a gap between these bounds, 
we can simply redefine the marginal distribution of $P$ over $\cX$ to have $P(\{w_{1}\} \times \cY) = 1$.
Note that with this distribution, $x_{i} = w_{1}$ for all $i$, with probability $1$, so that we clearly have $\hatn{m} = 1$ almost surely, 
and hence $\Bound{\hat{n}}{m}{\delta} = 1$ for all $m$.  As argued above, we have $\gamma(\F,1) \geq d$ for this space.
Therefore, $\max_{m \leq M} \gamma(\F,\Bound{\hat{n}}{m}{\delta}) \geq d$, while $\max_{m \leq M} \Bound{\hat{n}}{m}{\delta} \leq 1$, for all $M \in \nats$.
However, note that unlike the example constructed above for the disagreement coefficient,
the gap in this example could potentially be eliminated by replacing the distribution-free quantity $\gamma(\F,n)$
with a distribution-dependent complexity measure (e.g., an annealed VC entropy or a bracketing number for $\{ \DIS(\VS_{\F,S}) : S \in (\cX \times \cY)^{n} \}$).

\vskip 0.2in
\bibliography{S_PAC}

\end{document}